\documentclass{article}

\usepackage{microtype}
\usepackage{graphicx}
\usepackage{subfigure}
\usepackage{booktabs} % for professional tables

\usepackage{hyperref}

\usepackage[accepted]{icml2025}

\usepackage{amsmath}
\usepackage{amssymb}
\usepackage{mathtools}
\usepackage{amsthm}

\usepackage[capitalize,noabbrev]{cleveref}

\theoremstyle{plain}
\newtheorem{theorem}{Theorem}[section]
\newtheorem{proposition}[theorem]{Proposition}
\newtheorem{lemma}[theorem]{Lemma}

\theoremstyle{definition}

\theoremstyle{remark}

\usepackage[textsize=tiny]{todonotes}

\usepackage{algorithm}
\usepackage{algorithmic}
\usepackage{colortbl}
\usepackage{caption}
\DeclareMathOperator{\card}{card}
\DeclareMathOperator*{\argmin}{arg\,min}

\usepackage{rotating}
\usepackage{tablefootnote}
\usepackage[export]{adjustbox}
\usepackage{multirow}
\usepackage{cases}
\usepackage{booktabs}
\usepackage{amsthm}
\usepackage{amsmath}     % For advanced math typesetting
\usepackage{amssymb}     % For additional mathematical symbols
    \usepackage{graphicx}    % For including graphics
\usepackage{url}         % For proper URL formatting
\usepackage{amsthm}
\graphicspath{{images/}}
\usepackage[table]{xcolor}
\usepackage{hyperref}
\icmltitlerunning{Dual Feature Reduction for SGL}

\begin{document}

\twocolumn[
\icmltitle{Dual Feature Reduction for the Sparse-group Lasso and its Adaptive Variant}




\begin{icmlauthorlist}
\icmlauthor{Fabio Feser}{icl}
\icmlauthor{Marina Evangelou}{icl}
%\icmlauthor{Firstname5 Lastname5}{yyy}
%\icmlauthor{Firstname6 Lastname6}{sch,yyy,comp}
%\icmlauthor{Firstname7 Lastname7}{comp}
%\icmlauthor{Firstname8 Lastname8}{sch}
%\icmlauthor{Firstname8 Lastname8}{yyy,comp}
\end{icmlauthorlist}

\icmlaffiliation{icl}{Department of Mathematics, Imperial College London, London, UK}
\icmlcorrespondingauthor{Fabio Feser}{ff120@ic.ac.uk}

\icmlkeywords{Machine Learning, ICML}

\vskip 0.3in
]



\printAffiliationsAndNotice{}  % leave blank if no need to mention equal contribution

\begin{abstract}
The sparse-group lasso performs both variable and group selection, simultaneously using the strengths of the lasso and group lasso. It has found widespread use in genetics, a field that regularly involves the analysis of high-dimensional data, due to its sparse-group penalty, which allows it to utilize grouping information. However, the sparse-group lasso can be computationally expensive, due to the added shrinkage complexity, and its additional hyperparameter that needs tuning. This paper presents a novel feature reduction method, Dual Feature Reduction (DFR), that uses strong screening rules for the sparse-group lasso and the adaptive sparse-group lasso to reduce their input space before optimization, without affecting solution optimality. DFR applies two layers of screening through the application of dual norms and subdifferentials. Through synthetic and real data studies, it is shown that DFR drastically reduces the computational cost under many different scenarios.
\end{abstract}

\section{Introduction}
%High-dimensional datasets, where the number of features ($p$) is far greater than the number of observations ($n$) in a matrix $\mathbf{X}\in \mathbb{R}^{n\times p}$, are becoming increasingly common with the increased rate of data collection. Applying traditional regression approaches to high-dimensional data is troublesome; the ordinary least squares solution does not have a solution if $p>n$, as it requires evaluation of ($\mathbf{X}^\top \mathbf{X})^{-1}$, which is singular in this case \citep{Hastie2009TheEdition}. To alleviate this, models such as ridge \citep{Hoerl1970RidgeProblems}, the lasso \citep{Tibshirani1996RegressionLasso}, elastic-net \citep{Zou2005RegularizationNet}, and SLOPE \citep{Bogdan2015SLOPEAdaptiveOptimization} have been proposed and found increased use in the machine learning community \citep{Alaoui2015FastGuarantees, Lemhadri2021LassoNet:Sparsity,Michoel2018AnalyticNet, Thompson2023TheNetworks}. These models fall under the umbrella of \textit{shrinkage methods}, as the estimated coefficients are shrunk towards zero during optimization, overcoming the singularity issue of the data matrix. Of the many shrinkage methods, the lasso has found the most widespread use, due to its ability to shrink coefficients exactly to zero, performing \textit{variable selection}. 
High-dimensional datasets, where the number of features ($p$) is far greater than the number of observations ($n$) in a matrix $\mathbf{X}\in \mathbb{R}^{n\times p}$, are becoming increasingly common with the increased rate of data collection. To handle this, \textit{shrinkage methods}, such as the lasso \citep{Tibshirani1996RegressionLasso}, elastic-net \citep{Zou2005RegularizationNet}, and SLOPE \citep{Bogdan2015SLOPEAdaptiveOptimization} have been proposed and found increased use in the machine learning community \citep{Alaoui2015FastGuarantees, Michoel2018AnalyticNet,Lemhadri2021LassoNet:Sparsity, Thompson2023TheNetworks}. These methods shrink estimates towards zero during optimization, enabling \textit{variable selection}, to identify which features, $\beta\in\mathbb{R}^p$, have an association with the response $y\in \mathbb{R}^n$. 

%Variable selection allows a researcher to determine which features have an association with the response $y\in \mathbb{R}^n$. This is particularly useful in genetics, where a biological researcher may wish to discover a candidate list of genes with an association to a disease outcome. These genes can then be examined in more detail through biological experiments. However, genes are naturally found in pathways (groups of genes) and any regression approach that does not use this grouping information would not be making full use of the information available. Attempts to incorporate grouping information, leading to \textit{group selection}, can be found through methods such as the group lasso \citep{Yuan2006ModelVariables}, group SLOPE \citep{DamianBrzyskiAlexejGossmann2019GroupPredictors}, and group SCAD \citep{Guo2015ModelSCAD}. However, applying only group shrinkage can lead to issues with convergence and poor predictive performance, as the model is forced to keep all variables within an active group, leading to many noisy variables being part of the optimization process \citep{Feser2023Sparse-groupFDR-control,Simon2013}. 

In genetics, these methods help identify genes associated with disease outcomes. As genes are naturally found in groups (pathways), \textit{group selection} approaches have been proposed, that allow grouping information to be used, such as the group lasso \citep{Yuan2006ModelVariables}, group SLOPE \citep{DamianBrzyskiAlexejGossmann2019GroupPredictors}, and group SCAD \citep{Guo2015ModelSCAD}. Applying only group shrinkage can harm convergence and prediction, as all variables in an active group are retained, including noise variables \citep{Simon2013,Feser2023Sparse-groupFDR-control}. 

%The limitations of group-based regression models led to the development of sparse-group shrinkage models, such as the Sparse-group Lasso (SGL) \citep{Simon2013} and Sparse-group SLOPE (SGS) \citep{Feser2023Sparse-groupFDR-control}. These models apply shrinkage on both variables and groups to yield concurrent variable and group selection, allowing grouping information to be used without the constraint of every variable in a group being active. In particular, SGL has found increased popularity in applications in the machine learning \citep{Vidyasagar2014MachineCancer,Yogatama2014LinguisticCategorization} and healthcare \citep{Fang2015Bi-levelLasso, Peng2010RegularizedCancer,Simon2013} communities, due to its ability to shrink whole groups of variables to zero, as well as variables within active groups. It has been shown to consistently outperform the lasso and group lasso in terms of selection and predictive performance \citep{Simon2013,Feser2023Sparse-groupFDR-control}, showing that sparse-group models provide tangible benefits.

This limitation led to the development of sparse-group models, such as the Sparse-group Lasso (SGL) \citep{Simon2013} and Sparse-group SLOPE (SGS) \citep{Feser2023Sparse-groupFDR-control}. These models can shrink variables in active groups by applying shrinkage on both variables and groups to yield bi-level selection. SGL has found increased popularity in applications in the machine learning \citep{Vidyasagar2014MachineCancer,Yogatama2014LinguisticCategorization} and healthcare \citep{Peng2010RegularizedCancer,Simon2013,Fang2015Bi-levelLasso} communities. Sparse-group models have been shown to have tangible benefits by consistently outperforming the lasso and group lasso in selection and prediction tasks \citep{Simon2013,Feser2023Sparse-groupFDR-control}.

Suppose the variables sit in a grouping structure, with disjoint sets of variables $\mathcal{G}_1,\dots,\mathcal{G}_m$ of sizes $p_1,\ldots,p_m$. Then, SGL is a convex combination of the lasso and group lasso \citep{Simon2013}:
\begin{align}
    \hat\beta_\text{sgl}(\lambda) \in  &\argmin_{\beta\in \mathbb{R}^p}\{f(\beta) +\lambda \|\beta\|_\text{sgl} \}, \label{eqn:sgl_problem}\\
    \text{where}\;\; \|\beta\|_\text{sgl} &= \alpha \|\beta\|_1 +  (1-\alpha) \sum_{g=1}^m \sqrt{p_g}\|\beta^{(g)}\|_2,\label{eqn:sgl_problem_2}
\end{align}
such that $f$ is a differentiable and convex loss function, $\lambda>0$ defines the level of shrinkage, $\beta^{(g)}\in\mathbb{R}^{p_g}$ is the vector of coefficients in group $g$, and $\alpha \in [0,1]$. SGL has been extended to have adaptive shrinkage through the adaptive sparse-group lasso (aSGL) \citep{Poignard2020AsymptoticLasso,Mendez-Civieta2021AdaptiveRegression} (Section \ref{section:adap_sgl}).

\subsection{Feature Reduction Approaches for the Sparse-group Lasso}\label{section:sgl_other_methods}
The strengths of SGL come with increased computational cost, due to the additional shrinkage and the tuning of two hyperparameters. Typically, $\alpha$ is set subjectively (\citet{Simon2013} suggest $\alpha = 0.95$) and $\lambda$ is tuned using cross-validation along a path $\lambda_1\geq \ldots\geq \lambda_l\geq 0$. Algorithms such as the Least Angle Regression (LARS) approach calculate solutions for all possible values of $\lambda$, but are very sensitive to multicollinearity and scale quadratically, rendering their use in high-dimensional settings limited \citep{Efron2004LeastRegression}. Instead, feature reduction techniques, including screening rules, can help ease the cost of fitting a model along a path by discarding features before optimization that would have been inactive at the optimal solution. Whilst methods exist to discard observations \citep{Shibagaki2016SimultaneousModeling, Zhang2017ScalingReduction}, the focus of this paper is high-dimensional settings, in which discarding features is more impactful on computational savings.

Feature reduction techniques are either \textit{exact} or \textit{heuristic}. Exact methods strictly discard only inactive features but are conservative, while heuristic methods discard more features at the risk of violations. These violations are countered by checking the Karush–Kuhn–Tucker (KKT) conditions \citep{Kuhn1951NonlinearProgramming} and adding any offending features back into the optimization. Heuristic rules discard significantly more variables than exact rules, providing large computational savings \citep{tibshirani2010strong}.

Most exact methods follow the seminal Safe Feature Elimination (SAFE) framework \citep{Ghaoui2010SafeProblems}, which has been applied to the group lasso \citep{Bonnefoy2015DynamicGroup-Lasso} and SGL \citep{Ndiaye2016GAPLasso}. Other exact examples include the dome test \citep{Xiang2012FastCorrelations}, Dual Polytope Projections (DPP) \citep{Wang2013LassoProjection}, and Slores \citep{Wang2014ARegression}. The strong rule by \citet{tibshirani2010strong} provides a framework for applying heuristic reduction with single separable norms, which has been extended to non-separable \citep{Larsson2020a} and sparse-group norms \citep{Feser2024Screen}. Other examples include Sure Independence Screening (SIS) \citep{Fan2008SureSpace} and the Hessian rule \citep{Larsson2024TheRule}. 

Aside from the exact and heuristic categories, feature reduction techniques tend to follow three forms: \textit{static}, where the feature reduction occurs only once at the start \citep{Ghaoui2010SafeProblems, Xiang2011LearningDictionaries, Xiang2012FastCorrelations}, \textit{dynamic}, where reduction occurs iteratively \citep{Bonnefoy2015DynamicGroup-Lasso}, and \textit{sequential}, where information from the previous solution is used \citep{tibshirani2010strong,Larsson2020a,Larsson2024TheRule, Feser2024Screen}. 

%An exact feature reduction method for SGL was proposed by \citet{Ndiaye2016GAPLasso}, called the GAP safe rules. The method combines sequential and dynamic feature reduction and creates feasible regions in which active variables sit using the duality gap. GAP safe applies two layers of reduction, discarding inactive groups and inactive variables within active groups. Other SGL exact methods are Two-layer Feature Reduction (TLFre) \citep{Wang2014Two-LayerSets}, although this was shown not to be exact \citep{Ndiaye2015GAPModels}. A heuristic approach for SGL was introduced by \citet{Liang2022Sparsegl:Lasso}, called sparsegl, which is a strong rule that applies only group-level reduction. There have been other attempts to speed up SGL; \citet{Ida2019FastLasso} calculate approximate bounds for the inactive conditions derived in \citet{Simon2013}, and \citet{Li2022FastBiobank} derive a heuristic screening rule, but it is limited to multi-response Cox modeling.

An exact reduction method for SGL, called GAP safe, was proposed by \citet{Ndiaye2016GAPLasso} using the SAFE framework. GAP safe uses the duality gap to create feasible regions where the active variables sit and applies reduction on the groups and variables. Other reduction methods for SGL include Two-layer Feature Reduction (TLFre) (exact) \citep{Wang2014Two-LayerSets}, though it was shown not to be exact \citep{Ndiaye2015GAPModels}, and sparsegl (heuristic) \citep{Liang2022Sparsegl:Lasso}, which applies only group-level reduction. Additional speed-up attempts include using approximate bounds for inactive conditions \citep{Ida2019FastLasso} and a heuristic screening rule limited to multi-response Cox modeling \citep{Li2022FastBiobank}.
\begin{figure}[t]
\centering
  \includegraphics[width=1\linewidth,valign=t]{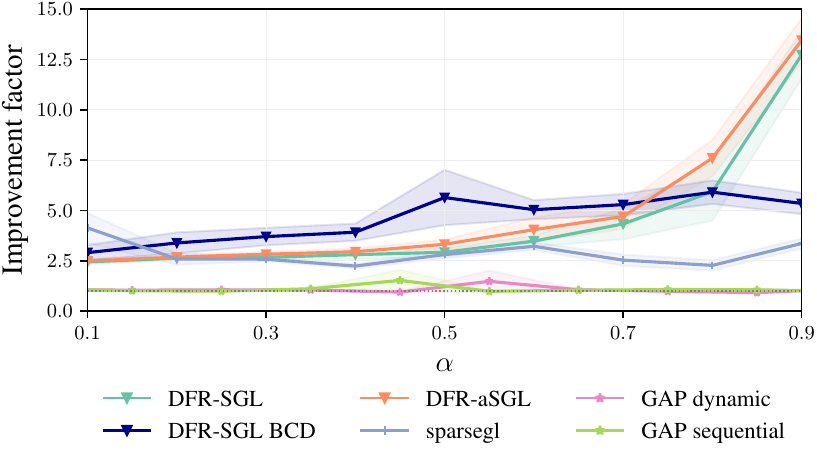}
  \captionof{figure}{The improvement factor (high is best), which measures by how many orders the screening has improved fitting time, shown for different SGL screening methods, as a function of $\alpha$, with $95\%$ confidence intervals. The data was generated under a linear model with even groups of sizes $20$ (Section \ref{section:results_synthetic}). The two GAP methods have very similar values so jitter was added.}
\label{fig:figure-group-3-if}
\end{figure}
\subsection{Contributions}
In this paper, we propose a new feature reduction method for SGL and adaptive SGL, \textit{Dual Feature Reduction} (DFR), which is based on the strong rule \citep{tibshirani2010strong} and the sparse-group screening framework \citep{Feser2024Screen}. DFR introduces the first bi-level strong screening rules for SGL and the first screening rules for aSGL. 

DFR applies two layers of screening, discarding inactive groups and inactive variables within active groups. By reducing the input dimensionality before optimization, DFR enables expanded tuning regimes to be performed, including concurrent tuning of $\lambda$ and $\alpha$. These benefits are achieved without affecting solution optimality. The computational efficiency of DFR increases the accessibility of SGL and aSGL models, encouraging wider adoption across fields. 

The GAP safe rules for SGL require computation of safe regions, which includes a radius and center, and the dual norm, as well as iterative screening and fitting. In contrast, DFR needs only the dual norm and screens only once at each path point, making it considerably less expensive, as evidenced by our results (Figure \ref{fig:figure-group-3-if}). 

DFR is described for SGL in Section \ref{section:sgl_feature_reduction} and then extended to aSGL in Section \ref{section:adap_sgl}. The proofs of the results presented in these sections are provided in Appendix \ref{appendix:sgl_theory} for SGL and Appendix \ref{appendix:adap_sgl_theory} for aSGL. DFR is applied to synthetic and real data in Sections \ref{sec:results} and \ref{section:real_data_results}, where it is found to be the state-of-the-art screening approach for SGL, outperforming other existing methods, while still achieving the optimal solution.

\section{Preliminaries}\label{section:theory}
%The proofs of the propositions presented in this section are provided in Appendices \ref{appendix:sgl_grp_theory}, \ref{appendix:sgl_var_theory} for SGL, and Appendix \ref{appendix:adap_sgl_theory} for aSGL.

\subsection{Problem Statement}
SGL is trained along a path of parameters $\lambda_1\geq \ldots \geq \lambda_l\geq 0$. The objective is to use the solution at $\lambda_k$ to generate a set of \textit{candidate variables} $\mathcal{C}_v(\lambda_{k+1}) \subset [p]:=\{1,\ldots,p\}$, that is a superset of the (unknown) set of active variables at $\lambda_{k+1}$, given by $\mathcal{A}_v(\lambda_{k+1}) := \{i \in [p]:\hat\beta_i(\lambda_{k+1})\neq 0\}$. The optimization at $\lambda_{k+1}$ (Equation \ref{eqn:sgl_problem}) is then calculated using only $\mathcal{C}_v(\lambda_{k+1})$. If the candidate set is a small proportion of the total input space, large computational savings are expected.

To generate the candidate variable set, we first generate a candidate group set (Section \ref{section:sgl_group_reduction}), which is then used as a basis for constructing the candidate variable set (Section \ref{section:sgl_variable_reduction}). This is done using the dual norm of SGL.

\subsection{Dual Norm}\label{section:sgl_dual_norm}
DFR requires evaluating the dual norm of SGL, defined as $\|z\|_\text{sgl}^*:= \sup\{z^\top x: \|x\|_\text{sgl}\leq 1\}.$ The SGL norm can be expressed in terms of the dual of the $\epsilon$-norm \citep{Ndiaye2016GAPLasso}, 
\begin{equation}\label{eqn:sgl_as_e_norm}
    \|\beta\|_\text{sgl}= \sum_{g=1}^m \tau_g\|\beta^{(g)}\|^*_{\epsilon_g},
\end{equation}
where $\tau_g= \alpha+(1-\alpha)\sqrt{p_g}$. The $\epsilon$-norm, $\|x\|_{\epsilon_g}$, applied to a group $g\in[m]$, is defined as the unique nonnegative solution $q$ of the equation \citep{Burdakov1988AProblems}
\begin{equation}\label{defn:e-norm}
\sum_{i=1}^{p_g} (|x_i| - (1-\epsilon_g)q)^2_+ = (\epsilon_g q)^2,  \;\; \epsilon_g = \frac{\tau_g-\alpha}{\tau_g}.
\end{equation}
Using this, by \citet{Ndiaye2016GAPLasso}, the dual norm of SGL applied to a group $g\in [m]$ can be formulated as
\begin{equation}\label{eqn:sgl_dual}
    \|\xi^{(g)}\|_\text{sgl}^* = \max_{g = 1,\ldots,m} \tau_g^{-1} \|\xi^{(g)}\|_{\epsilon_g}.
\end{equation}
\section{Dual Feature Reduction}\label{section:sgl_feature_reduction}
DFR is first derived for SGL (Section \ref{section:dfr_sgl}) and then extended to aSGL (Section \ref{section:adap_sgl}). DFR is summarised for SGL and aSGL in Table \ref{table:summary}.
\subsection{Sparse-group Lasso}\label{section:dfr_sgl}
\subsubsection{Group reduction}\label{section:sgl_group_reduction}
To generate a candidate group set, the KKT stationarity conditions \citep{Kuhn1951NonlinearProgramming} are used, providing conditions for an inactive group. For SGL, they are given by, for a group $g\in[m]$ at $\lambda_{k+1}$ (using the $\epsilon$-norm representation in Equation \ref{eqn:sgl_as_e_norm})
\begin{equation}\label{eqn:sgl_kkt}
\mathbf{0} \in \nabla_g f(\hat\beta(\lambda_{k+1}))+\tau_g\lambda_{k+1} \Theta_{g,k+1}, 
\end{equation}
where $\Theta_{g,k+1} = \partial \|\hat\beta(\lambda_{k+1})\|_{\epsilon_g}^*$ is the subgradient of the dual norm of the $\epsilon$-norm at $\lambda_{k+1}$. The subgradient for an inactive group $g$ (at zero) can be expressed by the unit ball of the dual norm \citep{Schneider2022TheEstimation}: 
\begin{equation*}
    \Theta^0_{g,k+1}:=\partial \|0\|^*_{\epsilon_g} = \left\{x\in\mathbb{R}^{p_g}: \|x\|_{\epsilon_g} \leq 1 \right\}.
\end{equation*} 
Plugging the unit ball into Equation \ref{eqn:sgl_kkt} and applying the $\epsilon$-norm, the subgradient can be canceled out, so the KKT conditions can be written as
\begin{equation}\label{eqn:sgl_group_screen_eq_1} 
\|\nabla_g f(\hat\beta(\lambda_{k+1}))\|_{\epsilon_g} =\tau_g \lambda_{k+1} \| \Theta^0_{g,k+1}\|_{\epsilon_g} \leq  \tau_g\lambda_{k+1}.
\end{equation}
If the gradient at $\lambda_{k+1}$ were available, we could exactly identify the active groups (Proposition \ref{propn:sgl_grp_screen_theoretical}). 
\begin{proposition}[Theoretical SGL group screening]\label{propn:sgl_grp_screen_theoretical}
    For SGL applied with any $\lambda_{k+1}, k\in [l-1]$, the candidate group set,
    \begin{equation*}
        \mathcal{C}_g(\lambda_{k+1}) = \{g\in [m]: \|\nabla_g f(\hat\beta(\lambda_{k+1}))\|_{\epsilon_g}  >  \tau_g\lambda_{k+1}\},
    \end{equation*} 
   is such that $\mathcal{C}_g(\lambda_{k+1}) = \mathcal{A}_g(\lambda_{k+1}):= \{g \in [m]:\|\hat\beta^{(g)}(\lambda_{k+1})\|_2\neq 0\}$.
\end{proposition}
However, as this is not possible in practice, an approximation $\mathcal{M}_g$ is required such that
\begin{equation}\label{eqn:approximation_problem}
     \|\nabla_g f(\hat\beta(\lambda_{k+1}))\|_{\epsilon_g}\leq \mathcal{M}_g.
\end{equation}
Then, the screening rule tests whether $\mathcal{M}_g \leq\tau_g\lambda_{k+1}$. If this is found to be true, it can be concluded that Equation \ref{eqn:sgl_group_screen_eq_1} holds and the group must be inactive.
An approximation can be found by assuming that the gradient is a Lipschitz function of $\lambda_{k+1}$ with respect to the $\epsilon$-norm,
\begin{equation}\label{eqn:sgl_group_lipschitz}
    \|\nabla_g f(\hat\beta(\lambda_{k+1})) - \nabla_g f(\hat\beta(\lambda_k))\|_{\epsilon_g} \leq \tau_g|\lambda_{k+1} - \lambda_k|,
\end{equation}
which is a similar assumption to that used in the lasso strong rule \citep{tibshirani2010strong}. Using the reverse triangle inequality gives
\begin{equation*}
\|\nabla_g f(\hat\beta(\lambda_{k+1}))\|_{\epsilon_g} \leq \underbrace{\|\nabla_g f(\hat\beta(\lambda_{k}))\|_{\epsilon_g} +\tau_g(\lambda_k - \lambda_{k+1})}_{=:\mathcal{M}_g},
\end{equation*}
yielding a suitable approximation $\mathcal{M}_g$. Therefore, the strong group screening rule for SGL (Proposition \ref{propn:sgl_grp_screen}) can be formulated by plugging $\mathcal{M}_g$ into Equation \ref{eqn:approximation_problem}: discard a group $g\in [m]$ if 
\begin{equation}\label{eqn:sgl_group_screen}
\|\nabla_gf(\hat\beta(\lambda_{k}))\|_{\epsilon_g} \leq \tau_g(2\lambda_{k+1} - \lambda_k).
\end{equation}
Since the Lipschitz assumption can fail, KKT checks (Section \ref{section:kkt_checks}) are performed to prevent violations.
\begin{proposition}[DFR-SGL group screening]\label{propn:sgl_grp_screen}
    For SGL applied with any $\lambda_{k+1}, k\in [l-1]$, assuming that 
    \begin{equation*}
      \|\nabla_g f(\hat\beta(\lambda_{k+1})) - \nabla_g f(\hat\beta(\lambda_k))\|_{\epsilon_g} \leq \tau_g|\lambda_{k+1} - \lambda_k|,
    \end{equation*}
    for all $g \in [m]$, then the candidate group set,
    \begin{align*}
        \mathcal{C}_g(\lambda_{k+1}) = \{&g\in [m]: \\
        &\|\nabla_g f(\hat\beta(\lambda_{k}))\|_{\epsilon_g} >\tau_g(2\lambda_{k+1} - \lambda_k)\},
    \end{align*} 
   is such that $\mathcal{A}_g(\lambda_{k+1}) \subset \mathcal{C}_g(\lambda_{k+1})$.
\end{proposition}
\subsubsection{Variable reduction}\label{section:sgl_variable_reduction}
Group screening reduces the input dimensionality, but further reduction is possible by applying a second screening layer to the variables in the candidate groups. For an inactive variable, $i\notin\mathcal{A}_v(\lambda_{k+1}), i\in\mathcal{G}_g$, the KKT conditions are (by Equation \ref{eqn:sgl_problem})
\begin{equation}\label{eqn:sgl_var_kkt_cond}
\mathbf{0} \in \nabla_i f(\hat\beta(\lambda_{k+1})) + \lambda_{k+1}  \alpha \Phi_{i,k+1}^0 + \lambda_{k+1} (1-\alpha)\Psi^{(g)}_{i,k+1}, 
\end{equation}
where $\Phi$ and $\Psi$ are the subgradients of the $\ell_1$ and $\ell_2$ norms respectively. For an active group, the subgradient of the $\ell_2$ norm is given by $\hat\beta_i^{(g)}/\|\hat\beta^{(g)}\|_2$, which vanishes for an inactive variable. So,
\begin{align}
    -\nabla_i f(\hat\beta(\lambda_{k+1})) &\in \lambda_{k+1}  \alpha \Phi_{i,k+1}^0  \nonumber \\
    \iff |\nabla_i f(\hat\beta(\lambda_{k+1}))|&\leq \lambda_{k+1}\alpha,\label{eqn:sgl_kkt_var}
\end{align}
where $\Phi^0_{i,k+1} = \{x\in\mathbb{R}:|x|\leq 1\}$. As before, knowledge of the gradient would lead to exact support recovery (Proposition \ref{propn:sgl_var_screen_theoretical}). Equation \ref{eqn:sgl_kkt_var} is similar to the strong screening rule for the lasso \citep{tibshirani2010strong}, scaled by $\alpha$. Hence, using a scaled version of Lipschitz assumption for the lasso, the variable screening rule for SGL is formulated in Proposition \ref{propn:sgl_var_screen}. 

\begin{proposition}[DFR-SGL variable screening]\label{propn:sgl_var_screen}
For SGL applied with any $\lambda_{k+1}, k\in [l-1]$, assuming that 
    \begin{equation*}
      |\nabla_i f(\hat\beta(\lambda_{k+1}))- \nabla_i f(\hat\beta(\lambda_{k}))| \leq \alpha(\lambda_{k} - \lambda_{k+1}),
    \end{equation*}
for all $i\in\mathcal{G}_g$ for $g\in\mathcal{A}_g(\lambda_{k+1})$, then the candidate variable set,
    \begin{align}
        \mathcal{C}_v(\lambda_{k+1}) = \{i&\in \mathcal{G}_g \;\text{for}\; g\in\mathcal{A}_g(\lambda_{k+1}):\nonumber\\
        &|\nabla_i f(\hat\beta(\lambda_{k}))| > \alpha (2\lambda_{k+1} -\lambda_k)\}\nonumber,
    \end{align}
    is such that $\mathcal{A}_v(\lambda_{k+1}) \subset \mathcal{C}_v(\lambda_{k+1})$.
\end{proposition}

To derive Proposition \ref{propn:sgl_var_screen}, knowledge of $\mathcal{A}_g(\lambda_{k+1})$ is required, but this is unknown. By Proposition \ref{propn:sgl_grp_screen}, $\mathcal{A}_g(\lambda_{k+1}) \subset \mathcal{C}_g(\lambda_{k+1})$, and so the candidate set is used in practice for applying Proposition \ref{propn:sgl_var_screen}. That is, the variable screening rule is applied to the candidate group set. Any violations caused by this replacement are checked for by the KKT checks, which are performed in any case (for any strong rule).

\subsubsection{KKT checks}\label{section:kkt_checks}
The screening rules of DFR use several Lipschitz assumptions (Propositions \ref{propn:sgl_grp_screen} and \ref{propn:sgl_var_screen}), as well as replacing the active group set by the candidate group set for the variable screening step (Section \ref{section:sgl_variable_reduction}). When these assumptions fail, the screening rules can incorrectly discard active variables. To protect against this, the KKT conditions are checked for each variable after screening. A KKT violation occurs for variable $i \in\mathcal{G}_g$ if 
\begin{equation}\label{eqn:sgl_final_kkt}
    |S(\nabla_i f(\hat\beta(\lambda_{k+1})), \lambda_{k+1} (1-\alpha)\sqrt{p_g})| > \lambda_{k+1}  \alpha,
\end{equation}
where $S(a,b) = \text{sign}(a)(|a|-b)_+$ is the soft-thresholding operator (see Appendix \ref{appendix:sgl_kkt_checks} for the derivation). A violating variable is added back into the optimization procedure (see Section \ref{section:algorithm}).

\subsubsection{Path start}\label{appendix:sgl_path_start}
When fitting SGL along a path of values, $\lambda_1\geq \ldots \geq \lambda_l \geq 0$, $\lambda_1$ is often chosen to be the exact point at which the first predictor becomes non-zero. By \citet{Ndiaye2016GapPenalties} and using the dual norm from Equation \ref{eqn:sgl_dual}, this value is given by 
\begin{equation*}
    \lambda_1 = \|\nabla f(0)\|^*_\text{sgl} =   \max_{g = 1,\ldots,m}\tau_g^{-1} \|\nabla_g f(0)\|_{\epsilon_g}.
\end{equation*}
\subsection{Algorithm}\label{section:algorithm}
The DFR algorithm is based on the sparse-group strong screening framework, proposed by \citet{Feser2024Screen}, and is shown in Algorithm \ref{alg:sgs_framework}. The algorithm has the following key steps for $\lambda_{k+1}$:

%DFR applies a layer of group screening, followed by variable screening on any remaining groups, to form the candidate variable set $\mathcal{C}_v$. This is combined with the previously active variables to form the \textit{optimization set}, $\mathcal{O}_v$, which is the input space to fit SGL on. Any variables outside the optimization set are set to zero. The KKT checks are then performed (Section \ref{section:kkt_checks}), with any violation variables added to the optimization set. This is repeated until no violations occur. 

\begin{enumerate}
    \item \textit{Group screening}: find $\mathcal{C}_g(\lambda_{k+1})$ using Proposition \ref{propn:sgl_grp_screen}.
    \item \textit{Variable screening}: find $\mathcal{C}_v(\lambda_{k+1})$ using Proposition \ref{propn:sgl_var_screen} for $i \in\mathcal{G}_g \setminus \mathcal{A}_v(\lambda_k), g \in \mathcal{C}_g(\lambda_{k+1})$.
    \item \textit{Optimization}: Compute $\hat\beta_{\mathcal{O}_v}(\lambda_{k+1})$ using the \textit{optimization set} $\mathcal{O}_v =\mathcal{C}_v(\lambda_{k+1}) \cup \mathcal{A}_v(\lambda_{k})$. Perform KKT checks to identify any violations (Section \ref{section:kkt_checks}) and add offending variables into $\mathcal{O}_v$. Repeat this step until no violations.
\end{enumerate}

The two main computational costs of the algorithm are the calculation of the solution, $\hat\beta$, and the evaluation of the $\epsilon$-norm. The former depends on the fitting algorithm, as this framework is applicable for any SGL fitting algorithm, with proximal and descent algorithms typically having complexities of $O(tp^2)$, for $t$ iterations \citep{https://doi.org/10.1002/wics.1602}. The latter has a worst-case cost of $O(p_g\log p_g)$ \citep{Ndiaye2016GAPLasso}. Tables \ref{tbl:comp-breakdown-1} and \ref{tbl:comp-breakdown-2} in the Appendix provide a computational breakdown of the components of DFR.
\subsection{Adaptive Sparse-group Lasso}\label{section:adap_sgl}
The \textit{Adaptive Sparse-group Lasso} (aSGL) applies adaptive shrinkage in a sparse-group setting, achieving the oracle property in a double-asymptotic framework, and has the norm \citep{Poignard2020AsymptoticLasso,Mendez-Civieta2021AdaptiveRegression}
\begin{equation}\label{eqn:adap_sgl_problem}
       \|\beta\|_\text{asgl} = \alpha \sum_{i=1}^p v_i |\beta_i| +  (1-\alpha) \sum_{g=1}^m w_g\sqrt{p_g}\|\beta^{(g)}\|_2,
\end{equation}
where $v_i$ and $w_g$ are adaptive weights (described in Appendix \ref{appendix:adap_sgl_weights}). aSGL has a less straightforward connection to the $\epsilon$-norm that allows for the derivation of screening rules (Proposition \ref{propn:adap_sgl_e_norm}).
\begin{proposition}\label{propn:adap_sgl_e_norm}
The aSGL norm (Equation \ref{eqn:adap_sgl_problem}) can be expressed as the $\epsilon$-norm by
\begin{align}
      &\|\beta\|_\text{asgl} = \sum_{g=1}^m \gamma_g\|\beta^{(g)}\|^*_{\epsilon_g'}, \; \text{where} \label{eqn:adap_sgl_as_e_norm}\\
      &\gamma_g = \alpha\|v^{(g)}\|_1 -\frac{\alpha}{\|\hat{\beta}^{(g)}\|_1}\sum_{i,j\in \mathcal{G}_g,i\neq j} v_j|\hat\beta_i|\nonumber\\
      &\;\;\;\;\;\;\;\;\;\;\;\;\;\;\;\;\;\;\;\;\;\;+ (1-\alpha)w_g\sqrt{p_g}, \nonumber\\
     &\epsilon'_g = \gamma_g^{-1}(1-\alpha)w_g\sqrt{p_g}. \nonumber
\end{align}
\end{proposition}
\begin{proof}[Proof of Proposition \ref{propn:adap_sgl_e_norm}]
Splitting up the summation term in the variable norm yields
\begin{align*}
     &\alpha \sum_{i=1}^p v_i |\beta_i| = \alpha\sum_{g=1}^m \sum_{i\in \mathcal{G}_g} v_i |\beta_i| \\
     &= \alpha\sum_{g=1}^m \|\beta^{(g)}\|_1\left(\|v^{(g)}\|_1- \frac{\sum_{i,j\in \mathcal{G}_g,i\neq j} v_j|\beta_i|}{\|\beta^{(g)}\|_1} \right).
\end{align*}
This allows the aSGL norm to be written in terms of the groups
\begin{align}
  &\|\beta\|_\text{asgl} =\sum_{g=1}^m\Biggr[\left( \|v^{(g)}\|_1 -\frac{\sum_{i,j\in \mathcal{G}_g,i\neq j} v_j|\beta_i|}{\|\beta^{(g)}\|_1} \right)\alpha\|\beta^{(g)}\|_1 \nonumber \\
  &\;\;\;\;\;\;\;\;\;\;\;\;\;\;\;\;\;\;\;\;\;\;+ (1-\alpha) w_g\sqrt{p_g}\|\beta^{(g)}\|_2\Biggr].\label{eqn:appendix_sgl_1} 
\end{align}
Setting
\begin{align*}
         \gamma_g &=\alpha\|v^{(g)}\|_1 -\frac{\alpha\sum_{i,j\in \mathcal{G}_g,i\neq j} v_j|\beta_i|}{\|\beta^{(g)}\|_1} + (1-\alpha)w_g\sqrt{p_g},\\
         \epsilon'_g &= \frac{(1-\alpha)w_g\sqrt{p_g}}{\gamma_g},
\end{align*}
allows Equation \ref{eqn:appendix_sgl_1} to be written in terms of the $\epsilon$-norm
\begin{equation*}
    \|\beta\|_\text{asgl} = \sum_{g=1}^m \gamma_g\|\beta^{(g)}\|_{\epsilon'_g}^*.
\end{equation*}
\end{proof}
Through this connection, the aSGL norm admits a direct link to SGL (Equation \ref{eqn:sgl_as_e_norm}), allowing the DFR-SGL rules to be used for aSGL by replacing $\tau_g$ with $\gamma_g$ and $\epsilon_g$ with $\epsilon'_g$ in the $\epsilon$-norm (Appendices \ref{appendix:adap_sgl_screen_rules_grp} and \ref{appendix:adap_sgl_screen_rules_var}). The connection is further sharpened theoretically: Lemma \ref{lemma:adap_gamma_g} guarantees that the $\gamma_g$ term in Equation \ref{eqn:adap_sgl_as_e_norm} will always exist, even under inactive groups. Lemma \ref{lemma:asgl_to_sgl} guarantees that the representation of aSGL as the $\epsilon$-norm correctly reduces to the SGL $\epsilon$-norm representation under constant weights.

\begin{lemma}\label{lemma:adap_gamma_g}
   Under an inactive group $g\notin \mathcal{A}_g$, i.e. $\beta^{(g)}\equiv \boldsymbol{0}$, the $\gamma_g$ term in Equation \ref{eqn:adap_sgl_as_e_norm} exists.
\end{lemma}
\begin{lemma}\label{lemma:asgl_to_sgl}
    Under $v\equiv \mathbf{1}$ and $w \equiv  \mathbf{1}$ in Equation \ref{eqn:adap_sgl_as_e_norm}, for each $g\in[m]$, $\gamma_g = \tau_g$ and $\epsilon'_g = \epsilon_g$.
\end{lemma}
Algorithm \ref{alg:sgs_framework} is also applicable for aSGL, using the corresponding aSGL equations (Algorithm \ref{alg:adap_sgl_framework}).

\subsubsection{KKT checks}\label{appendix:adap_sgl_kkt_checks}
The KKT checks for aSGL are also similar to those for SGL (Section \ref{section:kkt_checks}): a KKT violation occurs for a variable $i\in\mathcal{G}_g$ if
\begin{equation}\label{eqn:adap_sgl_kkt_checks}
    |S(\nabla_i f(\hat\beta(\lambda_{k+1})), \lambda_{k+1} (1-\alpha)w_g\sqrt{p_g})| > \lambda_{k+1}  v_i\alpha.
\end{equation}
\subsubsection{Path start}\label{appendix:adap_sgl_path}
To find the path start for aSGL, the dual norm cannot be used, since all groups are zero at this point. As a result, $\gamma_g$ exists only in limit for $\beta^{(g)} \equiv \boldsymbol{0}$ (Lemma \ref{lemma:adap_gamma_g}). A solution can instead be found using a similar approach to that of \citet{Simon2013} for SGL, where the point is found by solving the piecewise quadratic, for each $g \in [m]$,
\begin{equation*}
    \left\|S\left(X^{(g) \top}y/n, \lambda_g v^{(g)}\alpha\right)\right\|_2^2-p_gw_g^2(1-\alpha)^2 \lambda_g^2 = 0,
\end{equation*}
where $X^{(g)} \in \mathbb{R}^{n\times p_g}$ is the design matrix for only group $g$ and $v^{(g)} \in \mathbb{R}^{p_g}$ contains the penalty weights for the variables in group $g$. Then, choosing $\lambda_1 = \max_g \lambda_g$ gives the path start point.
\section{Numerical Results} \label{sec:results}
In this section, the efficiency and robustness of DFR is evaluated through the analysis of synthetic data that capture different data characteristics. As the purpose of screening rules is to reduce the dimensionality of the input space and, as a result, reduce the computational cost, the following two metrics are used for evaluation:
\begin{figure}[t]
  \centering
    \includegraphics[width=.845\columnwidth]{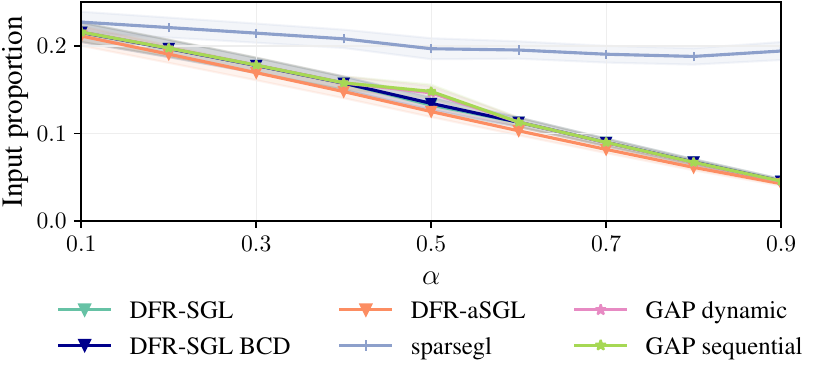}
  \caption{The input proportion for the screening methods applied to synthetic data, as a function of $\alpha$, with $95\%$ confidence intervals.}
  \label{fig:figure-group-3-ip}
\end{figure}

\begin{itemize}
\item \textit{Improvement factor = no screen time / screen time}, quantifies by how many orders the screening has improved the fitting time (high is best).  
\item \textit{Input proportion = $\mathcal{O}_v/p$}, measures how much of the input space was used in the optimization (low is best).
\end{itemize}
DFR is compared with the existing SGL screening rules sparsegl \citep{Liang2022Sparsegl:Lasso} and GAP safe \citep{Ndiaye2016GAPLasso}. sparsegl, in contrast to DFR, performs only a single layer of group screening. %\citet{Liang2022Sparsegl:Lasso} proposed the strong rule sparsegl, which in contrast to DFR, only performs a single layer of group screening. 
This rule is also based on the strong framework of \citet{tibshirani2010strong}, but uses a different Lipschitz assumption, which applies only the $\ell_2$ group norm, rather than the full SGL norm (as is the case for DFR). On the other hand, GAP safe is an exact feature reduction method for SGL that can be implemented dynamically or sequentially under linear regression \citep{Ndiaye2016GAPLasso}. GAP safe has many different implementation forms, and we are presenting the two versions that provided the best results in our studies. Appendix \ref{appendix:other_approaches} provides detailed descriptions of these two methods with Table \ref{table:summary} showing a summary of all rules considered.

%As with any safe rule, a safe region is constructed that contains the optimal dual solution.% in \citet{Ndiaye2016GAPLasso} it is taken as a sphere, but other regions can also be used. %GAP safe can be implemented dynamically or sequentially. 

 %Under $\alpha=\{0,1\}$, the DFR rules reduce to those of the (adaptive) group lasso and (adaptive) lasso respectively (Appendix \ref{appendix:lasso_glasso}).

Throughout the analyses, the SGL optimization for DFR and sparsegl is performed using the Adaptive Three Operator Splitting (ATOS) \citep{Pedregosa2018AdaptiveSplitting} algorithm, as it can be easily adapted for use with different sparse-group penalties (by simply swapping out the proximal operators), providing flexibility. However, DFR can be used with any fitting algorithm, including Block Coordinate Descent (BCD) \citep{Qin2013}. 
\begin{figure}[t]
  \centering
    \includegraphics[width=\columnwidth]{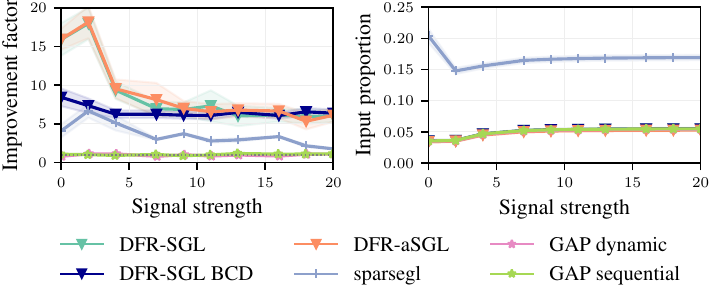}
  \caption{The improvement factor (left) and input proportion (right) for the screening methods applied to synthetic data, as a function of the signal strength, with $95\%$ confidence intervals. 
  }
  \label{fig:figure-group-3-i-ip-if}
\end{figure}
\subsection{Analysis of Synthetic Data}\label{section:results_synthetic}
The synthetic data was generated using a linear model, $y=\mathbf{X}\beta + \epsilon,$ where $\mathbf{X}\sim\mathcal{N}(\mathbf{0},\boldsymbol{\Sigma})\in \mathbb{R}^{200 \times 1000}$, with noise $\epsilon\sim\mathcal{N}(0,1)$ and where $\beta$ is a sparse vector with the signal sampled from $\mathcal{N}(0,4)$ (signal strength of zero). For $\mathbf{X}$, correlation was applied inside each group, such that $\Sigma_{i,j} = \rho = 0.3$ for $i$ and $j$ in the same group. The variables were placed in even groups of sizes $20$, with $0.2$ group sparsity proportion ($20\%$ of the groups are active) and $0.2$ variable sparsity proportion inside the active groups. The models were fit along a $50$-length path, starting at $\lambda_1 = \lambda_\text{max}$ (as defined by Section \ref{appendix:sgl_path_start} for SGL and Section \ref{appendix:adap_sgl_path} for aSGL), and terminating at $0.1\lambda_1$. Each simulation case was repeated $100$ times and the results are averaged across these repetitions, unless otherwise stated. Simulation and model implementation information can be found in Table \ref{appendix:set_up_info}.

\paragraph{Comparison to GAP Safe Rules} Comparing DFR to the GAP safe rules, under both varying $\alpha$ and signal strength, it is evident that the improvement factor is significantly superior for DFR compared to both the dynamic and sequential GAP rules (Figures \ref{fig:figure-group-3-if} and \ref{fig:figure-group-3-i-ip-if}). In fact, although the input proportion of DFR and GAP safe are of similar levels (Figures \ref{fig:figure-group-3-ip} and \ref{fig:figure-group-3-i-ip-if}), the cost of calculating safe regions appears to nullify any gain in dimensionality reduction. This comparison shows that the two reduction approaches (heuristic vs exact) arrive at very similar results (the screened sets), but DFR achieves this with greater computational efficiency.

For all values of $\alpha$, DFR considerably reduces the input space, with the screening efficiency a linearly decreasing function (Figure \ref{fig:figure-group-3-ip}). Under values of $\alpha$ close to zero, SGL is forced to pick more variables within a group as active, limiting the potential reduction of the input space. In such scenarios, the second screening layer becomes less crucial, evidenced by the similar performances of all approaches. Approaching the commonly used value of $\alpha=0.95$ shows the clear strengths of DFR. The screening methods are all relatively unaffected by the signal strength (Figure \ref{fig:figure-group-3-i-ip-if}).
\begin{figure}[t]
  \centering
    \includegraphics[width=\columnwidth,valign=t]{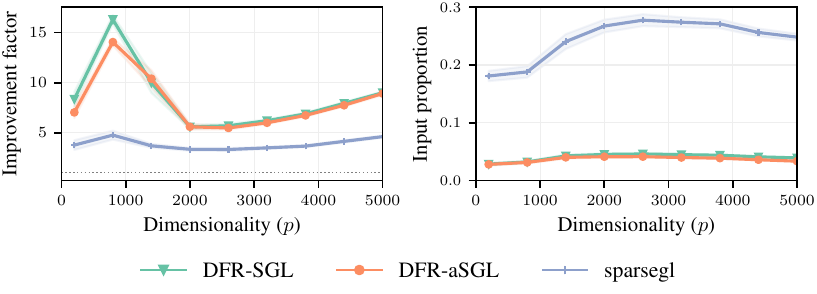}
  \caption{The improvement factor (left) and input proportion (right) for the strong rules applied to synthetic data, as a function of $p$, with $95\%$ confidence intervals. %The input proportion is shown in Figure \ref{fig:appendix_case_2_input_propn}.
  }
  \label{fig:figure-group-4-ip-if}
\end{figure}

\vspace{-12pt}
\begin{figure}[t]
  \centering    \includegraphics[width=\columnwidth,valign=t]{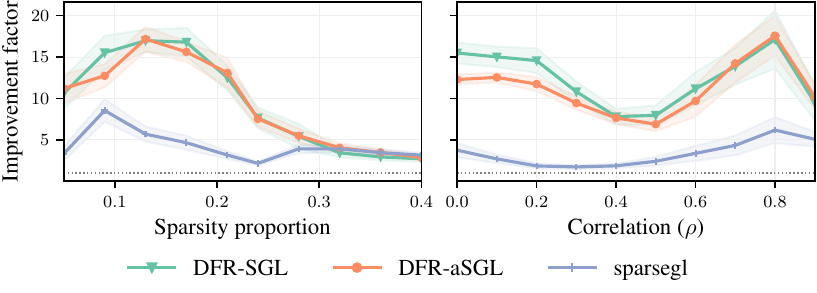}
  \caption{The improvement factor for the strong rules applied to synthetic data, as a function of the sparsity proportion (left) and data correlation (right), with $95\%$ confidence intervals. %The input proportion is shown in Figure \ref{fig:appendix_case_2_input_propn}.
  }
  \label{fig:figure-group-2-if}
\end{figure}
\paragraph{Comparison to BCD} The GAP safe rules are implemented using BCD while DFR uses ATOS. In Figures \ref{fig:figure-group-3-if} and \ref{fig:figure-group-3-i-ip-if}, DFR has also been implemented using BCD, demonstrating similar performance to ATOS, showing that the choice of fitting algorithm has little impact on the gain from screening.

\begin{table}[b]
  \centering
  \captionof{table}{The improvement factor for the strong rules applied to synthetic interaction data under the linear model, with standard errors. The parameters of the data were set as $p=400,n=80$, and $m=52$ groups of sizes in $[3,15]$. The interaction input dimensionality was $p_{O_2} = 2111$ and $p_{O_3}=7338$, with no interaction hierarchy imposed. The sparsity proportion of interaction variables was set to $0.3$ (with the same signal as the marginal effects). %normal variables).
  }
\label{tbl:results_inter_cv}
  \begin{tabular}{lrr}
    \toprule
     & \multicolumn{2}{c}{Interaction}   \\
    \cmidrule(lr){2-3}
                            Method & Order 2 & Order 3    \\
    \midrule
    DFR-aSGL                     &  $137.3\pm 12.0$ & $54.0\pm 10.7$   \\
    DFR-SGL                 & $44.3 \pm 2.4$  & $23.6\pm 3.1$     \\
    sparsegl                     & $7.4\pm 0.9$      & $1.2\pm 0.3$   \\
    \bottomrule
  \end{tabular}
\end{table}

%In Figure \ref{fig:case_5}, a spike is observed around $p=800$. In this case, the groups were fixed at size $20$, and so for each value of $p$ the group sizes, as a proportion of the input, are different. It is possible that $p=800$ represents a \textit{sweet spot} where the grouping structure favours bi-level screening. For small $p$, the grouping structure would dominate, so that discarding one of the few discrete chunks of groups would have a large influence on computational savings. On the other hand, as becomes large, the grouping structure becomes more irrelevant, as there are many small groups.

%\vspace{-0.5cm}
\paragraph{Increasing Dimensionality} The benefits of DFR over sparsegl are observed under varying $p$ (Figure \ref{fig:figure-group-4-ip-if}), peaking around $p = 1000$. With groups of fixed sizes of $20$, their relative size proportion to $p$ changes, suggesting an optimal grouping regime around $p=1000$ for screening. For small $p$, few large groups dominate, limiting the potential to screen out groups, while for large $p$, many small groups reduce the impact of group screening.

Additional explorations of increased dimensionality included the analysis of interaction terms (Table \ref{tbl:results_inter_cv} and Appendix \ref{appendix:interactions}). In this setting, all possible interactions of order 2 and 3 within each group were included. DFR provides large computational savings when fitting interactions, especially compared to sparsegl, which under order 3 interactions provides only marginal improvements. These savings make it more feasible for sparse-group models to be used in interaction detection problems. Such challenges are frequently seen in the field of genetics, where gene-gene and gene-environment relationships are useful discoveries \citep{DAngelo2009CombiningStudies,Zemlianskaia2022AInteractions}. 

%In these data with a $0.3$ active proportion (same signal as normal variables), resulting in an input dimensionality of $p_{O_2} = 2111, p_{O_3}=7338$ for orders 2 and 3. No interaction hierarchy is imposed in the data generation or model fitting.  

%Cross-validation (CV) is an important tool for tuning $\lambda$. However, due to its cost, $\alpha$ is often set manually, rather than included in a grid optimization scheme. Using DFR with 10-fold CV yielded computationally gains (Table \ref{tbl:results_cv}) that enable future tuning schemes for SGL to consider both $\alpha$ and $\lambda$, and aSGL to include the weight hyperparameters $\gamma_1,\gamma_2$.
\begin{figure}[t]
  \centering
    \includegraphics[width=\columnwidth,valign=t]{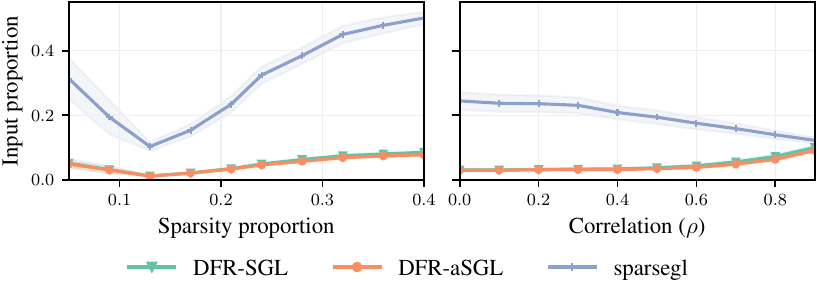}
  \caption{The input proportion for the strong rules applied to synthetic data, as a function of the sparsity proportion (left) and data correlation (right), with $95\%$ confidence intervals. %The input proportion is shown in Figure \ref{fig:appendix_case_2_input_propn}.
  }
  \label{fig:figure-group-2-ip}
\end{figure}
\begin{figure}[t]
  \centering
    \includegraphics[width=1\columnwidth, valign=t]{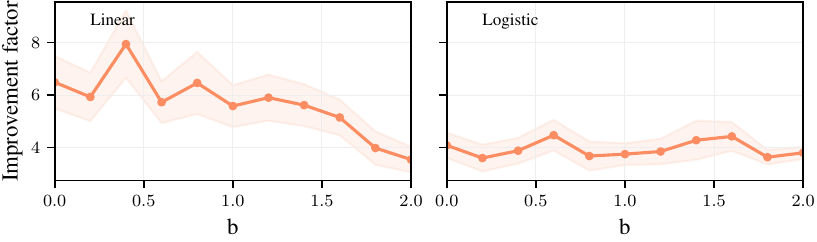}
  \caption{The improvement factor of DFR-aSGL under different weights $b_1=b_2$, shown for the linear (left) and logistic (right) models, with $95\%$ confidence intervals.}
  \label{fig:appendix_case_6}
\end{figure}
\paragraph{Robustness} For the remainder of the section, 
the variables were placed in $m=22$ uneven groups of sizes in $[3,100]$ to gain additional insights into the robustness of DFR. Earlier, DFR was found to be robust under varying signal strength and $\alpha$. We further observe that DFR is also robust to the data-generating parameters of signal sparsity (variable and group sparsity proportions varied together) and group correlation in $\mathbf{X}$ (Figures \ref{fig:figure-group-2-if} and \ref{fig:figure-group-2-ip}).

A clear benefit of DFR over sparsegl is observed under sparse signals. Screening rules generally have a greater impact as the signal becomes sparser. However, once the signal saturates, their effectiveness declines, leading to similar performance across approaches. Under varying correlation, DFR is more successful at reducing the input space when compared to sparsegl, especially under minor correlation. Under higher correlation, the models become less sparse, again resulting in reduced screening importance. 

Similar to the case with even groups, for uneven groups, DFR remains relatively unaffected by the signal strength and the choice of $\alpha$, consistently achieving effective reduction. (Figures \ref{fig:figure-group-6-if} and \ref{fig:figure-group-6-ip}). 

\paragraph{Hyperparameter Tuning and Cross-Validation}
The performance of DFR-aSGL was found to be robust under different values of hyperparameters $b_1$ and $b_2$ (Figure \ref{fig:appendix_case_6}), which are used to define the adaptive weights (Appendix \ref{appendix:adap_sgl_weights}).

The efficiency and robustness of DFR across different hyperparameters ($\alpha, b_1, b_2$) make it a promising tool for enabling approaches like cross-validation (CV) to tune all SGL and aSGL hyperparameters, which is rarely done in practice. Applying DFR with CV yields substantial computational savings (Table \ref{tbl:results_cv}). These findings highlight DFR’s value in facilitating expanded tuning regimes for SGL and aSGL.
\begin{figure}[t]
  \centering
    \includegraphics[width=\columnwidth]{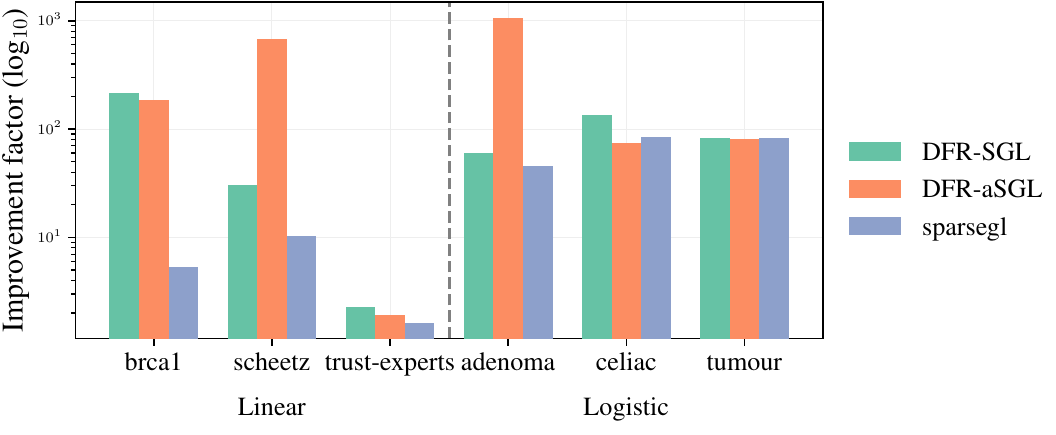}
  \caption{The improvement factor ($\log_{10}$ scale) of the strong rules applied to the six real datasets, split by model type. %All datasets, apart from \textit{trust-experts}, are high-dimensional.%Dataset information is given in Table \ref{tbl:appendix_real_dataset_info}
  }
  \label{fig:real_data_bar_standardized}
\end{figure}
\begin{figure}[t]
  \centering
    \includegraphics[width=\columnwidth]{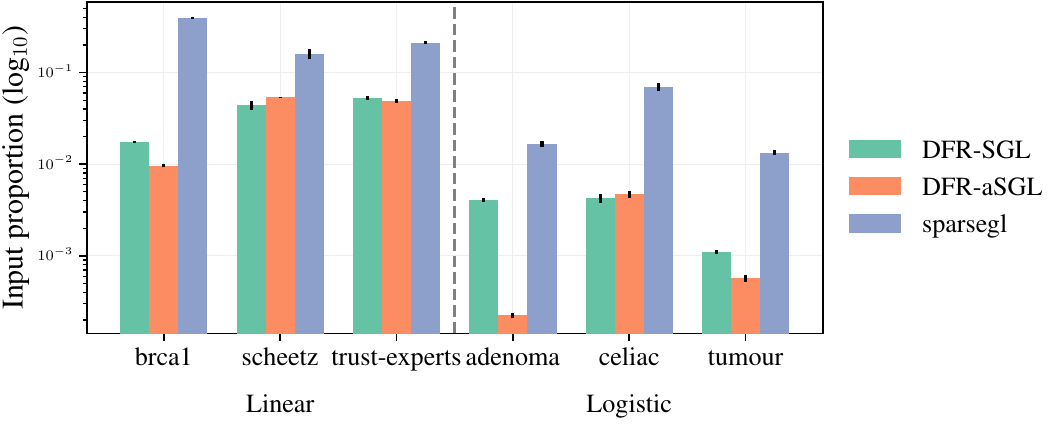}
  \caption{The input proportion ($\log_{10}$ scale) of the strong rules applied to the six real datasets, split by model type.}
  \label{fig:appendix_real_data_input_propn}
\end{figure}
\paragraph{Logistic Model}
DFR is also effective and robust for logistic models (Figures \ref{fig:figure-group-2-log-if}, \ref{fig:figure-group-2-log-ip}, \ref{fig:figure-group-6-log-if}, and \ref{fig:figure-group-6-log-ip}; see Appendix \ref{appendix:results_log} for further results and description of the data generation).

%The effectiveness of DFR is independence of the model used for generating the data, whether linear or logistic (Appendix \ref{appendix:results_log}).
%Investigations with data generated under a logistic model, showcase the effectiveness of DFR for such  data (Appendix \ref{appendix:results_log}). 

%The experiments were also performed under a logistic model, where similar performance results were observed, suggesting that the proposed DFR approaches are suitable for both linear and logistic regression modes (see Appendix \ref{appendix:results_log}).

\paragraph{KKT Violations}
KKT violations for DFR are very rare. Across all experiments with linear models, DFR-SGL had only a single KKT violation (Table \ref{tbl:appendix_other_sims_var}). Violations were more common for DFR-aSGL and sparsegl, but still infrequent. Note that DFR-aSGL violations refer to variable ones, and sparsegl to group ones, making it more likely to have a variable violation. The elevated number of KKT violations for sparsegl suggests that the group Lipschitz assumption of DFR-SGL is more robust.

One possible explanation for the increased KKT violations of DFR-aSGL lies in the role of the Lipschitz assumptions. Unlike SGL, the adaptive penalties in aSGL introduce additional dependencies on hyperparameters into the Lipschitz assumptions. This extra dependence on hyperparameters has the potential to lead to violations.
\begin{figure}[t]
  \centering
    \includegraphics[width=\columnwidth]{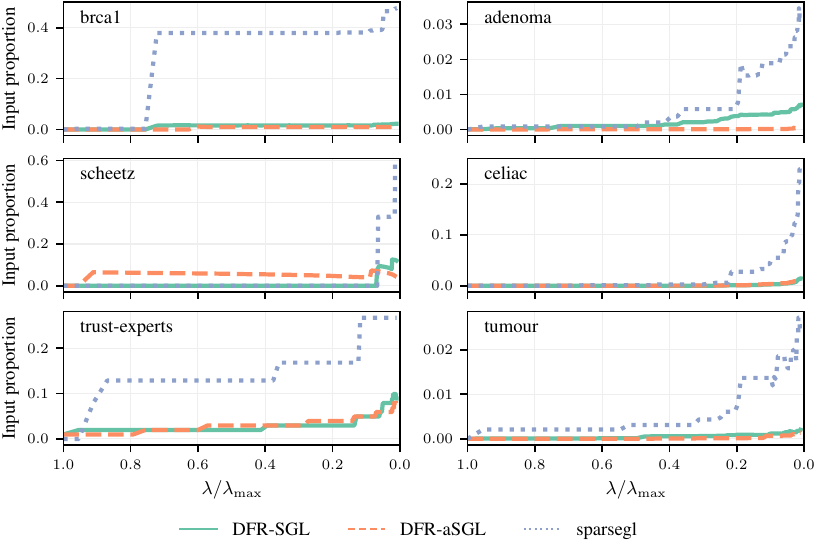}
  \caption{The input proportion as a function of the shrinkage path for the strong rules applied to the real datasets.  %For the other real datasets, see Figure \ref{fig:appendix_real_data_pathwise_other}.
  }
  \label{fig:real_data_pathwise}
\end{figure}
\section{Real Data Analysis}\label{section:real_data_results}
The efficiency of DFR is further evaluated through the analysis of six real datasets with different characteristics, including response type and dimensionality. Three of the datasets, \textit{brca1}, \textit{scheetz}, and \textit{trust-experts}, have continuous responses, so are fit using an SGL linear model. The former two were also analyzed with regards to screening rules in \citet{Larsson2024TheRule}, and the latter in \citet{Liang2022Sparsegl:Lasso}. 

The other three datasets, \textit{adenoma}, \textit{celiac}, and \textit{tumour}, have binary responses, so an SGL logistic model is used. The \textit{trust-experts} dataset is low-dimensional, and the other five are high-dimensional. The models were fit along a $100$-length path, terminating at $0.2\lambda_1$, where $\lambda_1$ generates the null model. More information on the datasets and model implementation is provided in Table \ref{appendix:set_up_info} and Appendix \ref{appendix:realdata}.

For all datasets, DFR outperforms sparsegl at reducing the computational cost (Figure \ref{fig:real_data_bar_standardized}) and input dimension (Figure \ref{fig:appendix_real_data_input_propn}), as well as keeping the input proportion low along the whole path (Figure \ref{fig:real_data_pathwise}). Despite being most useful for high-dimensional data, even in the case of low-dimensional data (\textit{trust-experts}), DFR improves fitting time. 

DFR-aSGL performs very well for \textit{scheetz} and \textit{adenoma}, improving the computational cost by over $600$ times. For the \textit{scheetz} dataset, the aSGL model had more difficulty converging without screening compared to SGL, so DFR-aSGL offered a greater advantage over DFR-SGL. For \textit{adenoma}, the active set for aSGL was smaller (Table \ref{tbl:appendix_real_data_var}), due to the increased penalization that comes with the adaptivity. However, despite the advantage of a smaller active set, we do still observe that DFR-aSGL was more efficient at reducing the optimization set, \textit{with respect to the active set}.

DFR is observed to aid in mitigating convergence issues for both SGL and aSGL (Table \ref{tbl:appendix_real_data_other}). Across all datasets, DFR encountered no failed convergences. In contrast, sparsegl did not converge at several path points for both \textit{adenoma} and \textit{scheetz}. As sparsegl only screens groups, when a group enters the optimization set, sparsegl is forced to fit with the full group, which can contain noise variables. Applying no screening led to SGL not converging for \textit{adenoma}, \textit{scheetz}, and \textit{tumour}. By drastically reducing the input space, convergence issues caused by large datasets are resolved, improving both computational cost and solution optimality.

\section{Discussion}\label{section:discussion}
A novel feature reduction method for the sparse-group lasso and adaptive sparse-group lasso, called \textit{Dual Feature Reduction} (DFR), has been introduced, derived using the dual norms of SGL and aSGL. DFR introduces the first bi-level strong screening rules for SGL and the first screening rules for aSGL. By applying two layers of reduction, DFR effectively reduces input dimensionality for optimisation and is computationally simpler than the GAP safe rules, which require iterative screening and fitting. In contrast, DFR screens only once per path point, so that it adds minimal computational overhead.

DFR first applies group-level screening, discarding inactive groups, followed by variable-level reduction, where inactive variables in active groups are removed. By discarding variables that are inactive at the optimal solution, DFR achieves significant computational savings, enabling the SGL family of models to scale more efficiently with increasing dimensionality and handle larger, more complex datasets. This gain comes at no cost, as the optimal solution is still achieved (Appendices \ref{appendix:lin_add_results}, \ref{appendix:results_log}, and \ref{appendix:real_add_results}). In fact, by reducing the input, instances were observed where DFR helped SGL and aSGL overcome convergence issues.

DFR proved robust across different data and model parameters, achieving drastic feature reduction under all scenarios considered. This consistently translated into large computational savings across both synthetic and real data. DFR outperformed all other screening approaches, establishing it as the state-of-the-art screening method for SGL and highlighting the benefit of bi-level screening.
%Only a single KKT violation was observed for DFR-SGL, while DFR-aSGL had only a negligible amount. 

\paragraph{Limitations}
Several assumptions are required to perform two layers of feature reduction for DFR. Propositions \ref{propn:sgl_grp_screen} and \ref{propn:sgl_var_screen} use Lipschitz assumptions which are consistent with the strong framework \citep{tibshirani2010strong}. Any breach of assumptions is guarded against by KKT checks. Only a single KKT violation occurred for SGL across all our simulations and only very infrequently for aSGL. These assumptions are a limitation of any strong rule, although DFR carries additional assumptions over other strong rules, which are necessary for the second layer of screening. 

\paragraph{Code} DFR is implemented in the \texttt{dfr} R package \citep{dfr-r-package}, available on CRAN.

\section*{Impact Statement} 
This paper aims to advance Machine Learning while ensuring no disadvantage to anyone. The proposed screening rules do not alter the solution but improve the accessibility of SGL and aSGL for researchers with limited computational resources.

\section*{Acknowledgements}
We would like to thank the anonymous reviewers for
their valuable comments. This work was supported by the Engineering and Physical Sciences Research Council (EPSRC) through the Modern Statistics and Statistical Machine Learning (StatML) CDT programme, grant no. EP/S023151/1.

\bibliography{references_new}
\bibliographystyle{icml2025}

%%%%%%%%%%%%%%%%%%%%%%%%%%%%%%%%%%%%%%%%%%%%%%%%%%%%%%%%%%%%%%%%%%%%%%%%%%%%%%%
%%%%%%%%%%%%%%%%%%%%%%%%%%%%%%%%%%%%%%%%%%%%%%%%%%%%%%%%%%%%%%%%%%%%%%%%%%%%%%%
% APPENDIX
%%%%%%%%%%%%%%%%%%%%%%%%%%%%%%%%%%%%%%%%%%%%%%%%%%%%%%%%%%%%%%%%%%%%%%%%%%%%%%%
%%%%%%%%%%%%%%%%%%%%%%%%%%%%%%%%%%%%%%%%%%%%%%%%%%%%%%%%%%%%%%%%%%%%%%%%%%%%%%%
\newpage
\appendix
\onecolumn
%\part{Appendix} % Start the appendix part
\renewcommand\thefigure{A\arabic{figure}}
\setcounter{figure}{0} 
\setcounter{table}{0}
\renewcommand{\thetable}{A\arabic{table}}
\setcounter{algorithm}{0}
\renewcommand{\thealgorithm}{A\arabic{algorithm}}

\icmltitle{Dual Feature Reduction for the Sparse-group Lasso and its Adaptive Variant: \\
Supplementary Materials}
\section{Sparse-group Lasso}
\subsection{Theory} \label{appendix:sgl_theory}
\subsubsection{Group reduction}\label{appendix:sgl_grp_theory}
\begin{proof}[Proof of Proposition \ref{propn:sgl_grp_screen_theoretical}]
To prove the two sets are equivalent, we need to prove that for any $g\in[m]$ and $k \in [l-1]$, $g\in \mathcal{A}_g(\lambda_{k+1}) \iff g \in \mathcal{C}_g(\lambda_{k+1})$. We instead prove the contrapositive:   $g \notin \mathcal{C}_g(\lambda_{k+1}) \iff g \notin \mathcal{A}_g(\lambda_{k+1}).$ So,
\begin{alignat*}{2}
    g \notin \mathcal{C}_g(\lambda_{k+1}) &\iff \|\nabla_g f(\hat\beta(\lambda_{k+1}))\|_{\epsilon_g} \leq \tau_g \lambda_{k+1}, \;\;&&\text{by definition of the candidate set}\\
    &\iff -\nabla_g f(\hat\beta(\lambda_{k+1}))  \in \tau_g\lambda_{k+1}\Theta^0_{g,k+1},\;\; &&\text{as}\; \Theta^0_{g,k+1} = \left\{x\in\mathbb{R}^{p_g}: \|x\|_{\epsilon_g} \leq 1 \right\}\\
    &\iff \mathbf{0} \in \nabla_g f(\hat\beta(\lambda_{k+1}))  + \tau_g\lambda_{k+1}\Theta^0_{g,k+1}\\
    &\iff g\notin \mathcal{A}_g(\lambda_{k+1}), \;\; &&\text{by the KKT conditions (Equation \ref{eqn:sgl_kkt})}.
\end{alignat*}
\end{proof}

\begin{proof}[Proof of Proposition \ref{propn:sgl_grp_screen}]
To prove the candidate set is a superset of the active set, we need to prove that for any $g\in[m]$ and $k \in [l-1]$, $g\in \mathcal{A}_g(\lambda_{k+1}) \implies g \in \mathcal{C}_g(\lambda_{k+1})$. We instead prove the contrapositive:   $g \notin \mathcal{C}_g(\lambda_{k+1}) \implies g \notin \mathcal{A}_g(\lambda_{k+1}).$ First, we rewrite the Lipschitz assumption as (using the reverse triangle inequality)
\begin{align}
    \|\nabla_g f(\hat\beta(\lambda_{k+1}))\|_{\epsilon_g} - \|\nabla_g f(\hat\beta(\lambda_k))\|_{\epsilon_g}  &\leq \|\nabla_g f(\hat\beta(\lambda_{k+1})) - \nabla_g f(\hat\beta(\lambda_k))\|_{\epsilon_g}\leq \tau_g|\lambda_{k+1} - \lambda_k|\nonumber \\ \implies  \|\nabla_g f(\hat\beta(\lambda_{k+1}))\|_{\epsilon_g} &\leq \|\nabla_g f(\hat\beta(\lambda_{k}))\|_{\epsilon_g} + \tau_g|\lambda_{k+1} - \lambda_k|. \label{eqn:appendix_propn_grp}
\end{align}
Now, as $g \notin \mathcal{C}_g(\lambda_{k+1})$, 
\begin{equation*}
    \|\nabla_g f(\hat\beta(\lambda_{k}))\|_{\epsilon_g} \leq \tau_g(2\lambda_{k+1} - \lambda_k).
\end{equation*}
Plugging this into Equation \ref{eqn:appendix_propn_grp} yields
\begin{alignat*}{2}
     &\|\nabla_g f(\hat\beta(\lambda_{k+1}))\|_{\epsilon_g} \leq \tau_g(2\lambda_{k+1} - \lambda_k) + \tau_g|\lambda_{k+1} - \lambda_k|\\
    \implies   &\|\nabla_g f(\hat\beta(\lambda_{k+1}))\|_{\epsilon_g} \leq \tau_g\lambda_{k+1}\\
    \implies &-\nabla_g f(\hat\beta(\lambda_{k+1}))  \in \tau_g\lambda_{k+1}\Theta^0_{g,k+1},\;\; &&\text{as}\; \Theta^0_{g,k+1} = \left\{x\in\mathbb{R}^{p_g}: \|x\|_{\epsilon_g} \leq 1 \right\}\\
    \implies &\mathbf{0} \in \nabla_g f(\hat\beta(\lambda_{k+1}))  + \tau_g\lambda_{k+1}\Theta^0_{g,k+1}\\
    \implies &g\notin \mathcal{A}_g(\lambda_{k+1}), \;\; &&\text{by the KKT conditions (Equation \ref{eqn:sgl_kkt})}.
\end{alignat*}
\end{proof}

\subsubsection{Variable reduction}\label{appendix:sgl_var_theory}
\begin{proposition}[Theoretical SGL variable screening]\label{propn:sgl_var_screen_theoretical}
    For SGL applied with any $\lambda_{k+1}, k\in [l-1]$, the candidate variable set, 
    \begin{equation*}
        \mathcal{C}_v(\lambda_{k+1}) = \{i\in \mathcal{G}_g \;\text{for}\; g\in\mathcal{A}_g(\lambda_{k+1}): |\nabla_i f(\hat\beta(\lambda_{k+1}))|> \lambda_{k+1}\alpha \},
    \end{equation*} 
    is such that $\mathcal{C}_v(\lambda_{k+1}) = \mathcal{A}_v(\lambda_{k+1})$.
\end{proposition}
\begin{proof}[Proof of Proposition \ref{propn:sgl_var_screen_theoretical}]
The proof strategy is similar to that of Proposition \ref{propn:sgl_grp_screen_theoretical}.
To prove the two sets are equivalent, we need to prove that for any $i\in \mathcal{G}_g$ such that $g\in\mathcal{A}_g$, and $k \in [l-1]$, $i\in \mathcal{A}_v(\lambda_{k+1}) \iff i\in \mathcal{C}_v(\lambda_{k+1})$. We instead prove the contrapositive: $i \notin \mathcal{C}_v(\lambda_{k+1}) \iff i \notin \mathcal{A}_v(\lambda_{k+1}).$ So,
\begin{alignat*}{2}
    i \notin \mathcal{C}_v(\lambda_{k+1}) &\iff 
    |\nabla_i f(\hat\beta(\lambda_{k+1}))|\leq  \lambda_{k+1}\alpha, \;\;&&\text{by definition of the candidate set}\\
    &\iff -\nabla_v f(\hat\beta(\lambda_{k+1}))  \in \lambda_{k+1}  \alpha \Phi_{i,k+1}^0,\;\; &&\text{as}\; \Phi^0_{i,k+1} = \left\{x\in\mathbb{R}: |x| \leq 1 \right\},\\
   & &&\text{for} \;\;i\in \mathcal{G}_g,g\in\mathcal{A}_g(\lambda_{k+1})\\
    &\iff \mathbf{0} \in \nabla_v f(\hat\beta(\lambda_{k+1}))  +  \lambda_{k+1}  \alpha \Phi_{i,k+1}^0\\
    &\iff i\notin \mathcal{A}_v(\lambda_{k+1}), \;\; &&\text{by the KKT conditions (Equation \ref{eqn:sgl_var_kkt_cond})}.
\end{alignat*}
\end{proof}
\begin{proof}[Proof of Proposition \ref{propn:sgl_var_screen}]
The proof strategy is similar to that of Proposition \ref{propn:sgl_grp_screen}. To prove the candidate set is a superset of the active set, we need to prove that for any $i\in \mathcal{G}_g$ such that $g\in\mathcal{A}_g$, and $k \in [l-1]$, $i\in \mathcal{A}_v(\lambda_{k+1}) \implies i\in \mathcal{C}_v(\lambda_{k+1})$. We instead prove the contrapositive: $i \notin \mathcal{C}_v(\lambda_{k+1}) \implies i \notin \mathcal{A}_v(\lambda_{k+1}).$ First, we rewrite the Lipschitz assumption as (using the reverse triangle inequality)
\begin{equation}\label{eqn:appendix_propn_var}
   |\nabla_i f(\hat\beta(\lambda_{k+1}))| \leq |\nabla_i f(\hat\beta(\lambda_{k}))| + \alpha|\lambda_{k+1} - \lambda_k|. 
\end{equation}
Now, as $i \notin \mathcal{C}_v(\lambda_{k+1})$, 
\begin{equation*}
    |\nabla_i f(\hat\beta(\lambda_{k}))| \leq \alpha(2\lambda_{k+1} - \lambda_k).
\end{equation*}
Plugging this into Equation \ref{eqn:appendix_propn_var} yields
\begin{alignat*}{2}
    &|\nabla_i f(\hat\beta(\lambda_{k+1}))| \leq \alpha\lambda_{k+1}\\
    \implies &-\nabla_i f(\hat\beta(\lambda_{k+1}))  \in \alpha\lambda_{k+1}\Phi^0_{i,k+1},\;\; &&\text{as}\; \Phi^0_{i,k+1} = \left\{x\in\mathbb{R}: |x| \leq 1 \right\}\\
    \implies &\mathbf{0} \in \nabla_i f(\hat\beta(\lambda_{k+1}))  + \alpha\lambda_{k+1}\Phi^0_{i,k+1}\\
    \implies &i\notin \mathcal{A}_v(\lambda_{k+1}), \;\; &&\text{by the KKT conditions (Equation \ref{eqn:sgl_kkt})}.
\end{alignat*}
\end{proof}
\subsection{KKT Checks}\label{appendix:sgl_kkt_checks}
To determine whether a variable $i\in\mathcal{G}_g$ has been correctly discarded, the KKT stationarity conditions are checked. Equation \ref{eqn:sgl_var_kkt_cond} describes the condition under which a variable $i \in \mathcal{G}_g$ is inactive. Without specifying whether the group $g$ is inactive, this can be rewritten as (by the definition of $\Phi^0_{i,k+1}$)
\begin{equation}\label{eqn:appendix_sgl_kkt_1}
    |\nabla_i f(\hat\beta(\lambda_{k+1}))+ \lambda_{k+1} (1-\alpha)\Psi^{(g)}_{i,k+1}| \leq \lambda_{k+1}  \alpha,
\end{equation}
where $\Psi^{(g)}_{k+1} = \{x\in \mathbb{R}^{\sqrt{p_g}}: \|x\|_2\leq 1\}$ is the subgradient of the $\ell_2$ norm. To satisfy Equation \ref{eqn:appendix_sgl_kkt_1}, the unknown subdifferential, $\Psi^{(g)}_{i,k+1}$, is taken to be the minimum possible value. For $x\in \Psi^{(g)}_{k+1}$, we have that
\begin{align*}
    \|x\|_2\leq 1 &\implies  \sqrt{p_g}\|x\|_2\leq \sqrt{p_g}\\
    &\implies  \|x\|_1\leq \sqrt{p_g} \;\; \text{by the inequality}\;\;\|x\|_1 \leq \sqrt{p_g}\|x\|_2\\
    &\implies |x_i|\leq \sqrt{p_g}.
\end{align*}
Hence, the values in the subdifferential are bounded by $\sqrt{p_g}$. We consider the following scenarios for Equation \ref{eqn:appendix_sgl_kkt_1}:
\begin{enumerate}
    \item $\nabla_i f(\hat\beta(\lambda_{k+1})) > \lambda_{k+1}(1-\alpha)\sqrt{p_g}$: choose $x_i = -\sqrt{p_g}$.
    \item  $\nabla_i f(\hat\beta(\lambda_{k+1})) < -\lambda_{k+1}(1-\alpha)\sqrt{p_g}$: choose $x_i = \sqrt{p_g}$.
    \item $\nabla_i f(\hat\beta(\lambda_{k+1})) \in [-\lambda_{k+1}(1-\alpha) \sqrt{p_g},\lambda_{k+1}(1-\alpha) \sqrt{p_g}]$: choose $y_i =\frac{\nabla_i f(\hat\beta(\lambda_{k+1}))}{\lambda_{k+1}(1-\alpha)\sqrt{p_g}}$.
\end{enumerate}
This allows Equation \ref{eqn:appendix_sgl_kkt_1} to be expressed using the soft-thresholding operator as
\begin{equation*}
      |S(\nabla_i f(\hat\beta(\lambda_{k+1})), \lambda_{k+1} (1-\alpha)\sqrt{p_g})| \leq \lambda_{k+1}  \alpha.
\end{equation*}
A similar derivation can be found in \citet{Simon2013} to derive conditions to check whether a group is active for SGL.
\subsection{Algorithm}\label{appendix:dfr-sgl-algorithm}
\begin{algorithm}[H]
   \caption{Dual Feature Reduction (DFR) for SGL}
   \label{alg:sgs_framework}
\begin{algorithmic}
   \STATE {\bfseries Input:} $(\lambda_1,\ldots,\lambda_l) \in \mathbb{R}^l$, $\mathbf{X}\in \mathbb{R}^{n\times p}, y\in \mathbb{R}^n, \alpha \in [0,1]$
\STATE compute $\hat\beta(\lambda_1)$ using Equation \ref{eqn:sgl_problem}
\FOR{$k=1$ {\bfseries to} $l-1$}
\STATE $\mathcal{C}_g(\lambda_{k+1})$ $\leftarrow$ candidate groups from Proposition \ref{propn:sgl_grp_screen}
\STATE $\mathcal{C}_v(\lambda_{k+1})$ $\leftarrow$ candidate variables from Proposition \ref{propn:sgl_var_screen} for $i \in\mathcal{G}_g \setminus \mathcal{A}_v(\lambda_k), g \in \mathcal{C}_g(\lambda_{k+1})$
\STATE $\mathcal{O}_v \leftarrow \mathcal{C}_v(\lambda_{k+1}) \cup \mathcal{A}_v(\lambda_{k})$\hfill $\blacktriangleright$ Optimization set
\STATE compute $\hat\beta_i(\lambda_{k+1}), i\in\mathcal{O}_v$, using Equation \ref{eqn:sgl_problem}
\STATE $\mathcal{K}_v \leftarrow$ variable KKT violations for $i\notin \mathcal{O}_v$, using Equation \ref{eqn:sgl_final_kkt} \hfill $\blacktriangleright$ KKT check
\WHILE{$\card(\mathcal{K}_v)> 0$}
\STATE $\mathcal{O}_v \leftarrow \mathcal{O}_v\cup \mathcal{K}_v$ \hfill $\blacktriangleright$ Optimization set
\STATE compute $\hat\beta_i(\lambda_{k+1}),i\in\mathcal{O}_v$, using Equation \ref{eqn:sgl_problem}
\STATE $\mathcal{K}_v \leftarrow$ variable KKT violations for $i\notin \mathcal{O}_v$ using Equation \ref{eqn:sgl_final_kkt}\hfill $\blacktriangleright$ KKT check
\ENDWHILE
\ENDFOR
 \STATE {\bfseries Output:} $\hat\beta_\text{sgl}(\lambda_1), \ldots, \hat\beta_\text{sgl}(\lambda_l) \in \mathbb{R}^{p}$
\end{algorithmic}
\end{algorithm}

\subsection{Reduction to (Adaptive) Lasso and (Adaptive) Group Lasso}\label{appendix:lasso_glasso}
Under $\alpha = 1$, SGL reduces to the lasso. In this case, no group screening occurs and the variable screening rule reduces to the lasso strong rule \citep{tibshirani2010strong}:
\begin{equation*}
    |\nabla_i f(\hat\beta(\lambda_{k}))| \leq 2\lambda_{k+1} - \lambda_k.
\end{equation*}
Under $\alpha=0$, SGL reduces to the group lasso. Under this scenario, the group screening reduces to the group lasso strong rule \citep{tibshirani2010strong}: 
\begin{equation*}
     \|\nabla_g f(\hat\beta(\lambda_{k}))\|_2  \leq \sqrt{p_g} (2\lambda_{k+1} - \lambda_k),
\end{equation*}
and no variable screening is performed. For aSGL, the rules reduce to the adaptive lasso and adaptive group lasso:
\begin{alignat*}{3}
 \text{Adaptive lasso:}& \;|\nabla_i f(\hat{\beta}(\lambda_{k}))| \leq v_i (2\lambda_{k+1} - \lambda_k) &\implies \;&\hat{\beta}_i(\lambda_{k+1}) = 0.\\
      \text{Adaptive group lasso:}& \;\|\nabla_g f(\hat{\beta}(\lambda_{k}))\|_{\epsilon'_{g,1}} \leq w_g\sqrt{p_g}(2\lambda_{k+1} - \lambda_k) &\implies \;&\hat{\beta}^{(g)}(\lambda_{k+1}) \equiv \mathbf{0},
\end{alignat*}
where $\epsilon'_{g,1}$ denotes the $\epsilon$-norm under $\epsilon'_g = 1$ (Equations \ref{defn:e-norm} and \ref{eqn:adap_sgl_as_e_norm}).
\newpage 

\section{Adaptive Sparse-group Lasso}
\subsection{Derivation of the Connection to \texorpdfstring{$\epsilon$-norm}{epsilon-norm}}
\label{appendix:adaptive_sgl_derivation}
\begin{proof}[Full proof of Proposition \ref{propn:adap_sgl_e_norm}]
The adaptive SGL is given by
\begin{equation*}
  \|\beta\|_\text{asgl} = \alpha \sum_{i=1}^p v_i |\beta_i| +  (1-\alpha) \sum_{g=1}^m w_g\sqrt{p_g}\|\beta^{(g)}\|_2.    
\end{equation*}
The aim is to link this norm to the $\epsilon$-norm, in a similar way to SGL:
\begin{equation*}
      \|\beta\|_\text{sgl} = \sum_{g=1}^m(\alpha+(1-\alpha)\sqrt{p_g})\|\beta^{(g)}\|^*_{\epsilon_g}.
\end{equation*}
Splitting up the summation term in the adaptive lasso norm yields
\begin{align*}
     \alpha \sum_{i=1}^p v_i |\beta_i| &= \alpha\sum_{g=1}^m \sum_{i\in \mathcal{G}_g} v_i |\beta_i| \\
     &= \alpha\sum_{g=1}^m \left(\sum_{j\in \mathcal{G}_g} v_j\sum_{i \in \mathcal{G}_g}|\beta_i| - \sum_{i,j\in \mathcal{G}_g,i\neq j} v_j|\beta_i| \right) \\
     &=\alpha\sum_{g=1}^m \left(\sum_{j\in \mathcal{G}_g} v_j\sum_{i \in \mathcal{G}_g}|\beta_i| - \frac{\sum_{i,j\in \mathcal{G}_g,i\neq j} v_j|\beta_i|}{\sum_{i \in \mathcal{G}_g}|\beta_i|}\sum_{i \in \mathcal{G}_g}|\beta_i| \right)\\
     &=\alpha\sum_{g=1}^m 
 \sum_{i \in \mathcal{G}_g}|\beta_i|\left(\sum_{j\in \mathcal{G}_g} v_j- \frac{\sum_{i,j\in \mathcal{G}_g,i\neq j} v_j|\beta_i|}{\sum_{i\in \mathcal{G}_g} |\beta_i|} \right) \\
     &= \alpha\sum_{g=1}^m \|\beta^{(g)}\|_1\left(\|v^{(g)}\|_1- \frac{\sum_{i,j\in \mathcal{G}_g,i\neq j} v_j|\beta_i|}{\|\beta^{(g)}\|_1} \right).
\end{align*}
Hence
\begin{align}
  \|\beta\|_\text{asgl} &= \alpha \sum_{i=1}^p v_i |\beta_i| +  (1-\alpha) \sum_{g=1}^m w_g\sqrt{p_g}\|\beta^{(g)}\|_2 \nonumber \\
  &=\sum_{g=1}^m\left[\left( \|v^{(g)}\|_1 -\frac{\sum_{i,j\in \mathcal{G}_g,i\neq j} v_j|\beta_i|}{\|\beta^{(g)}\|_1} \right)\alpha\|\beta^{(g)}\|_1 + (1-\alpha) w_g\sqrt{p_g}\|\beta^{(g)}\|_2\right].\label{eqn:appendix_sgl_1_2} 
\end{align}
Setting
\begin{equation*}
         \gamma_g =\alpha\|v^{(g)}\|_1 -\frac{\alpha\sum_{i,j\in \mathcal{G}_g,i\neq j} v_j|\beta_i|}{\|\beta^{(g)}\|_1} + (1-\alpha)w_g\sqrt{p_g},
\end{equation*}
simplifies Equation \ref{eqn:appendix_sgl_1_2} to
\begin{equation}\label{eqn:appendix_sgl_2_2}
     \|\beta\|_\text{asgl} =\sum_{g=1}^m \gamma_g\left[\left(\frac{\gamma_g - (1-\alpha)w_g\sqrt{p_g}}{\gamma_g} \right)\|\beta^{(g)}\|_1 +\left(\frac{(1-\alpha)w_g\sqrt{p_g}}{\gamma_g}\right)\|\beta^{(g)}\|_2\right].
\end{equation}
Setting
\begin{equation*}
    \epsilon'_g = \frac{(1-\alpha)w_g\sqrt{p_g}}{\gamma_g},
\end{equation*}
allows Equation \ref{eqn:appendix_sgl_2_2} to be written in terms of the $\epsilon$-norm
\begin{equation*}
    \|\beta\|_\text{asgl} =\sum_{g=1}^m \gamma_g\left[(1-\epsilon'_g)\|\beta^{(g)}\|_1 +\epsilon'_g\|\beta^{(g)}\|_2\right] = \sum_{g=1}^m \gamma_g\|\beta^{(g)}\|_{\epsilon'_g}^*.
\end{equation*}
\end{proof}
\subsubsection{Lemma proofs for the connection to the \texorpdfstring{$\epsilon$-norm}{epsilon-norm}}\label{appendix:adap_sgl_additional_propertiese}
\begin{proof}[Proof of Lemma \ref{lemma:adap_gamma_g}]
Under $\beta^{(g)}\equiv \mathbf{0}$ for a group $g \notin \mathcal{A}_g$, the middle term in $\gamma_g$ becomes 
\begin{equation*}
\lim_{\beta^{(g)}\rightarrow \mathbf{0}} \left(\frac{\alpha\sum_{i,j\in \mathcal{G}_g,i\neq j} v_j|\beta_i|}{\|\beta^{(g)}\|_1}\right) = \frac{\alpha(p_g-1)}{p_g}\sum_{i=1}^{p_g}v_i,
\end{equation*}
so that $\gamma_g$ still exists. This can be observed by using L'H\^{o}pital's rule and noting that for $i \in \mathcal{G}_g$,
\begin{equation*}
    \frac{\partial}{\partial \beta_i}\sum_{i\neq j} v_j|\beta_i| = \sum_{i\neq j} v_j, \;\;  \frac{\partial}{\partial \beta_i}\|\beta^{(g)}\|_1 = 1.
\end{equation*}    
\end{proof}
\begin{proof}[Proof of Lemma \ref{lemma:asgl_to_sgl}]
    Under $v\equiv \mathbf{1}$ and $w\equiv \mathbf{1}$, note that
\begin{align*}
    \gamma_g &= \alpha \left(p_g - \frac{\sum_{i,j\in \mathcal{G}_g,i\neq j} v_j|\beta_i|}{\|\beta^{(g)}\|_1}\right) + (1-\alpha)\sqrt{p_g}\\ &= \alpha \left(p_g - \frac{(p_g - 1)\|\beta^{(g)}\|_1}{\|\beta^{(g)}\|_1}\right) + (1-\alpha)\sqrt{p_g}\\
    &= \alpha + (1-\alpha)\sqrt{p_g} = \tau_g.
\end{align*}
To understand the cross summation term, note that we are summing over each $\beta$ term $p_g - 1$ times, as the matching indices are removed, that is (for ease of notation, we consider $\mathcal{G}_1$ so that the indexing here is reset from 1)
\begin{align*}
    \sum_{i,j\in \mathcal{G}_1,i\neq j} v_j|\beta_i| &= |\beta_{1}|v_2 + \ldots + |\beta_1|v_{p_1} + |\beta_2|v_1 + \ldots + |\beta_2|v_{p_1}+ \ldots +|\beta_{p_1}|v_{p_1 - 1} \\
    &=(p_1-1)|\beta_1| + \ldots + (p_1-1)|\beta_{p_1}|, \; \text{by setting} \; v_j = 1, \forall j\in\mathcal{G}_1, \;\text{for SGL}\\
    & = (p_1 - 1)\sum_{i \in \mathcal{G}_1}|\beta_i| = (p_1 - 1)\|\beta^{(1)}\|_1.
\end{align*}
Hence, using $w_g = 1$ and $\tau_g = \alpha+(1-\alpha)\sqrt{p_g}$,
\begin{align*}
    \epsilon_g' = \frac{(1-\alpha)w_g\sqrt{p_g}}{\gamma_g} = \frac{(1-\alpha)\sqrt{p_g}}{\tau_g} = \frac{\tau_g - \alpha}{\tau_g} = \epsilon_g.
\end{align*}
\end{proof} 
\subsection{Theory}\label{appendix:adap_sgl_theory}
\subsubsection{Group screening}\label{appendix:adap_sgl_screen_rules_grp}
To derive the group screening rule for aSGL, we compare the formulations of SGL and aSGL in terms of the $\epsilon$-norm (Equations \ref{eqn:sgl_as_e_norm} and \ref{eqn:adap_sgl_as_e_norm}):
\begin{equation*}
 \|\beta\|_\text{sgl} = \sum_{g=1}^m \tau_g\|\beta^{(g)}\|^*_{\epsilon_g}, \;\; \|\beta\|_\text{asgl} = \sum_{g=1}^m \gamma_g\|\beta^{(g)}\|^*_{\epsilon_g'}.
\end{equation*}
Therefore, the derivation for the group screening rule for aSGL is identical to that of SGL (Section \ref{section:sgl_group_reduction}) replacing $\tau_g$ with $\gamma_g$ and $\|\cdot\|_{\epsilon_g}$ with $\|\cdot\|_{\epsilon_g'}$. The group screening rule is given by: discard a group $g$ if
\begin{equation}\label{eqn:adap_sgl_grp_screen_final}
     \|\nabla_g f(\hat{\beta}(\lambda_{k}))\|_{\epsilon_g'} \leq \gamma_g(2\lambda_{k+1} - \lambda_k),
\end{equation}
and is formalized in Propositions \ref{propn:adap_sgl_grp_screen_theoretical} and \ref{propn:adap_sgl_grp_screen}.
 
\begin{proposition}[Theoretical aSGL group screening]\label{propn:adap_sgl_grp_screen_theoretical}
    For aSGL applied with any $\lambda_{k+1}, k\in [l-1]$, the candidate group set, 
    \begin{equation*}
        \mathcal{C}_g(\lambda_{k+1}) = \{g\in [m]: \|\nabla_g f(\hat\beta(\lambda_{k+1}))\|_{\epsilon_g'}  >  \gamma_g\lambda_{k+1}\},
    \end{equation*}
    is such that $\mathcal{C}_g(\lambda_{k+1}) = \mathcal{A}_g(\lambda_{k+1})$.
\end{proposition}
\begin{proof}
The proof is identical to that of Proposition \ref{propn:sgl_grp_screen_theoretical} replacing $\tau_g$ with $\gamma_g$ and $\|\cdot\|_{\epsilon_g}$ with $\|\cdot\|_{\epsilon_g'}$ (see Appendix \ref{appendix:sgl_grp_theory}).
\end{proof}
\begin{proposition}[DFR-aSGL group screening]\label{propn:adap_sgl_grp_screen}
    For aSGL applied with any $\lambda_{k+1}, k\in [l-1]$, assuming that 
    \begin{equation*}
      \|\nabla_g f(\hat\beta(\lambda_{k+1})) - \nabla_g f(\hat\beta(\lambda_k))\|_{\epsilon_g'} \leq \gamma_g|\lambda_{k+1} - \lambda_k|,
    \end{equation*}
    for all $g\in[m]$, then the candidate group set,
   \begin{equation*}
       \mathcal{C}_g(\lambda_{k+1}) = \{g\in [m]: \|\nabla_g f(\hat\beta(\lambda_{k}))\|_{\epsilon_g'} > \gamma_g(2\lambda_{k+1} - \lambda_k)\},
   \end{equation*} 
   is such that $\mathcal{A}_g(\lambda_{k+1}) \subset \mathcal{C}_g(\lambda_{k+1})$.
\end{proposition}
\begin{proof}
    The proof is identical to that of Proposition \ref{propn:sgl_grp_screen} replacing $\tau_g$ with $\gamma_g$ and $\|\cdot\|_{\epsilon_g}$ with $\|\cdot\|_{\epsilon_g'}$ (see Appendix \ref{appendix:sgl_grp_theory}).
\end{proof}
\subsubsection{Variable screening}\label{appendix:adap_sgl_screen_rules_var}
The construction of the variable screening rule for aSGL is very similar to that of SGL (Section \ref{section:sgl_variable_reduction}). The KKT stationary conditions for aSGL for an inactive variable in an active group are (in comparison to Equation \ref{eqn:sgl_kkt_var} for SGL) 
\begin{equation*}
   -\nabla_i f(\hat\beta(\lambda_{k+1})) \in \lambda_{k+1}  \alpha v_i\Phi_{i,k+1}^0.
\end{equation*}
Therefore, the derivation of the rule is identical, replacing $\alpha$ with $\alpha v_i$. The variable screening rule is given by: discard a variable $i$ if
\begin{equation}\label{eqn:adap_sgl_var_screen_final}
     |\nabla_i f(\hat{\beta}(\lambda_{k}))| \leq \alpha v_i(2\lambda_{k+1} - \lambda_k),
\end{equation}
and is formalized in Propositions \ref{propn:adap_sgl_var_screen_theoretical} and \ref{propn:adap_sgl_var_screen}.
\begin{proposition}[Theoretical aSGL variable screening]\label{propn:adap_sgl_var_screen_theoretical}
    For aSGL applied with any $\lambda_{k+1}, k\in [l-1]$, the candidate variable set,
    \begin{equation*}
        \mathcal{C}_v(\lambda_{k+1}) = \{i\in \mathcal{G}_g \;\text{for}\; g\in\mathcal{A}_g(\lambda_{k+1}):  |\nabla_i f(\hat\beta(\lambda_{k+1}))|> \lambda_{k+1}\alpha v_i \},
    \end{equation*} 
    is such that $\mathcal{C}_v(\lambda_{k+1}) = \mathcal{A}_v(\lambda_{k+1})$.
\end{proposition}
\begin{proof}
The proof is identical to that of Proposition \ref{propn:sgl_var_screen_theoretical} replacing $\alpha$ with $\alpha v_i$ (see Appendix \ref{appendix:sgl_var_theory}).
\end{proof}
\begin{proposition}[DFR-aSGL variable screening]\label{propn:adap_sgl_var_screen}
    For aSGL applied with any $\lambda_{k+1}, k\in [l-1]$, assuming that 
    \begin{equation*}
      |\nabla_i f(\hat\beta(\lambda_{k+1}))- \nabla_i f(\hat\beta(\lambda_{k}))| \leq \alpha v_i(\lambda_{k} - \lambda_{k+1}),
    \end{equation*}
    for all $i\in\mathcal{G}_g$ for $g\in\mathcal{A}_g(\lambda_{k+1})$, then the variable candidate set, 
    \begin{equation*}
        \mathcal{C}_v(\lambda_{k+1}) = \{i\in \mathcal{G}_g \;\text{for}\; g\in\mathcal{A}_g(\lambda_{k+1}):   |\nabla_i f(\hat\beta(\lambda_{k}))| > \alpha v_i (2\lambda_{k+1} -\lambda_k) \},
    \end{equation*}
    is such that $\mathcal{A}_v(\lambda_{k+1}) \subset \mathcal{C}_v(\lambda_{k+1})$.
\end{proposition}
\begin{proof}
The proof is identical to that of Proposition \ref{propn:sgl_var_screen} replacing $\alpha$ with $\alpha v_i$ (see Appendix \ref{appendix:sgl_var_theory}).
\end{proof}
\subsection{Choice of Adaptive Weights}\label{appendix:adap_sgl_weights}
The adaptive weights are chosen according to \citet{Mendez-Civieta2021AdaptiveRegression} as
\begin{equation*}
    v_i = \frac{1}{|q_{1i}|^{b_1}}, w_g = \frac{1}{\|q_1^{(g)}\|_2^{b_2}},
\end{equation*}
where $q_1$ is the first principal component from performing principal component analysis on $\mathbf{X}$ and $b_1,b_2$ are chosen by the user, often in the range $[0,2]$. The weights are shown for $b_1=b_2=0.1$ in Figure \ref{fig:appendix_adaptive_weights}.

\begin{figure}[H]
  \centering
    \includegraphics[width=\columnwidth]{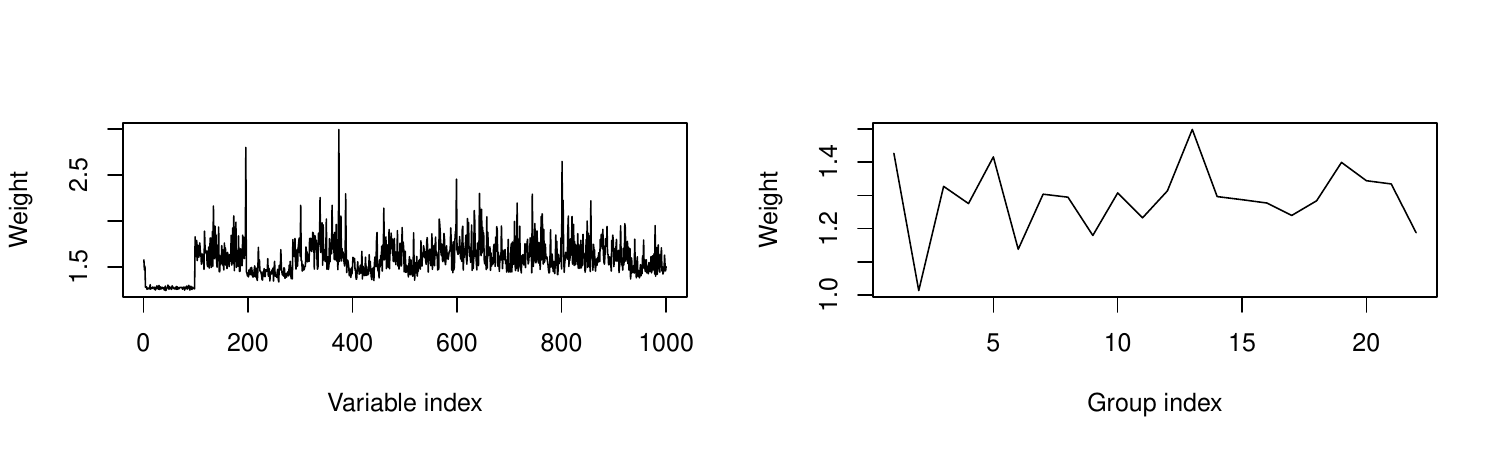}
  \caption{The weights, $(v,w)$, for aSGL, used in Figure \ref{fig:figure-group-2-if} (right), where $p=1000, n=200, m=22, \rho=0.3, b_1=b_2=0.1$, and $\alpha=0.95$.}
  \label{fig:appendix_adaptive_weights}
\end{figure}

\subsection{Algorithm} \label{appendix:asgl_algorithm}
\begin{algorithm}[!h]
   \caption{Dual Feature Reduction (DFR) for aSGL}
   \label{alg:adap_sgl_framework}
\begin{algorithmic}
   \STATE {\bfseries Input:} $(\lambda_1,\ldots,\lambda_l) \in \mathbb{R}^l, \mathbf{X}\in \mathbb{R}^{n\times p}, y\in \mathbb{R}^n, \alpha \in [0,1]$
\STATE compute $\hat\beta(\lambda_1)$ using Equation \ref{eqn:sgl_problem}, replacing the SGL norm with Equation \ref{eqn:adap_sgl_problem}
\FOR{$k=1$ {\bfseries to} $l-1$}
\STATE $\mathcal{C}_g(\lambda_{k+1})$ $\leftarrow$ candidate groups from Equation \ref{eqn:adap_sgl_grp_screen_final}
\STATE $\mathcal{C}_v(\lambda_{k+1})$ $\leftarrow$ candidate variables from Equation \ref{eqn:adap_sgl_var_screen_final} for $i \in\mathcal{G}_g \setminus \mathcal{A}_v(\lambda_k), g \in \mathcal{C}_g(\lambda_{k+1})$
\STATE $\mathcal{O}_v \leftarrow \mathcal{C}_v(\lambda_{k+1}) \cup \mathcal{A}_v(\lambda_{k})$\hfill $\blacktriangleright$ Optimization set
\STATE compute $\hat\beta_i(\lambda_{k+1}), i\in\mathcal{O}_v$, using Equation \ref{eqn:sgl_problem}, replacing the SGL norm with Equation \ref{eqn:adap_sgl_problem} 
\STATE $\mathcal{K}_v \leftarrow$ variable KKT violations for $i\notin \mathcal{O}_v$, using Equation \ref{eqn:adap_sgl_kkt_checks} \hfill $\blacktriangleright$ KKT check
\WHILE{$\card(\mathcal{K}_v)> 0$}
\STATE $\mathcal{O}_v \leftarrow \mathcal{O}_v\cup \mathcal{K}_v$ \hfill $\blacktriangleright$ Optimization set
\STATE compute $\hat\beta_i(\lambda_{k+1}),i\in\mathcal{O}_v$, using Equation \ref{eqn:sgl_problem}, replacing the SGL norm with Equation \ref{eqn:adap_sgl_problem}
\STATE $\mathcal{K}_v \leftarrow$ variable KKT violations for $i\notin \mathcal{O}_v$, using Equation \ref{eqn:adap_sgl_kkt_checks} \hfill $\blacktriangleright$ KKT check
\ENDWHILE
\ENDFOR
 \STATE {\bfseries Output:} $\hat\beta_\text{asgl}(\lambda_1), \ldots, \hat\beta_\text{asgl}(\lambda_l) \in \mathbb{R}^{p}$
\end{algorithmic}
\end{algorithm}

\newpage 

\section{Competitive Feature Reduction Approaches}\label{appendix:other_approaches}
\begin{table}[H]
\centering
  \caption{A summary of the four screening rules for SGL considered.}
    \label{table:summary}
 \resizebox{\textwidth}{!}{
 \begin{tabular}{llll}
\toprule
     & &\multicolumn{2}{c}{Rules (discard if true)}   \\
    \cmidrule(lr){3-4}  
                          Method&  Type& Variable & Group    \\
    \midrule
    DFR-aSGL     &Heuristic                & $|\nabla_i f(\hat{\beta}(\lambda_{k}))| \leq \alpha v_i (2\lambda_{k+1} - \lambda_k)$&$ \|\nabla_g f(\hat{\beta}(\lambda_{k}))\|_{\epsilon_g'} \leq \gamma_g(2\lambda_{k+1} - \lambda_k)$  \\
    DFR-SGL  &Heuristic           &$|\nabla_i f(\hat{\beta}(\lambda_{k}))| \leq \alpha (2\lambda_{k+1} - \lambda_k)$  &$\|\nabla_g f(\hat{\beta}(\lambda_{k}))\|_{\epsilon_g} \leq \tau_g(2\lambda_{k+1} - \lambda_k)$  \\
    sparsegl &Heuristic                    &\multicolumn{1}{c}{-}&  $    \| S(\nabla_g f(\hat\beta(\lambda_k)), \lambda_k\alpha)\|_2 \leq \sqrt{p_g}(1-\alpha)(2\lambda_{k+1} - \lambda_{k})$  \\
    GAP safe&Exact & $|X_i^\top \Theta_c| + r\|X_i\|_2 < \tau $ & $\mathcal{T}_g <(1-\alpha)\sqrt{p_g}$\\
    \bottomrule
  \end{tabular}
  }
\end{table}

\paragraph{sparsegl} sparsegl is a screening rule proposed by \citet{Liang2022Sparsegl:Lasso} and performs a single layer of group screening. The rule is based on the strong screening framework \citep{tibshirani2010strong} and the first order condition derived in \citet{Simon2013}, i.e., that a group $g\in [m]$ is inactive if
\begin{equation*}
    \| S(\nabla_g f(\hat\beta(\lambda_{k+1})), \lambda_{k+1}\alpha)\|_2\leq \sqrt{p_g}(1-\alpha)\lambda_{k+1}.
\end{equation*}
As the gradient at $k+1$ is not available, the following Lipschitz assumption on the $\ell_2$ norm is used:
\begin{equation*}
    \| S(\nabla_g f(\hat\beta(\lambda_{k+1})), \lambda_{k+1}\alpha) -  S(\nabla_g f(\hat\beta(\lambda_{k})), \lambda_{k}\alpha)\|_2 \leq \sqrt{p_g}(1-\alpha)|\lambda_{k+1} - \lambda_k|.
\end{equation*}
This leads to the sparsegl screening rule (via the triangle inequality): discard a group $g$ if
\begin{equation*}
     \| S(\nabla_g f(\hat\beta(\lambda_k)), \lambda_k\alpha)\|_2 \leq \sqrt{p_g}(1-\alpha)(2\lambda_{k+1} - \lambda_{k}).
\end{equation*}
This screening rule uses a different Lipschitz assumption at the group-level (DFR: Equation \ref{eqn:sgl_group_lipschitz}), which in turn leads to a different group-level rule (DFR: Equation \ref{eqn:sgl_group_screen}). Our Lipschitz assumption is more consistent with the work of \citet{tibshirani2010strong}, as the assumption is with regards to the dual norm of the full SGL norm, rather than just the group component.

\paragraph{GAP Safe} An exact feature reduction method for SGL was proposed in \citet{Ndiaye2016GAPLasso} under linear regression. The approach makes use of the subdifferential inclusion equation of Fermat's rule \citep{Bauschke2017ConvexSpaces}: 
\begin{equation*}
\mathbf{X}^\top \hat\Theta^{(\lambda,\|\cdot\|_{\text{sgl}})} \in \partial \|\cdot\|_{\text{sgl}}(\hat\beta^{(\lambda,\|\cdot\|_{\text{sgl}})}),
\end{equation*}
where $\hat\Theta$ is the solution to the dual formulation of Equation \ref{eqn:sgl_problem}. Using this, exact (theoretical) rules are derived to determine which variables and groups are inactive at the optimal solution. The rules are theoretical as they rely on $\hat\Theta^{\lambda,\|\cdot\|_{\text{sgl}}}$, which is not available in practice. Instead, a safe region is constructed that contains the optimal dual solution; in \citet{Ndiaye2016GAPLasso} it is taken as a sphere, but other regions can also be used (such as domes). Due to the strict requirements on these safe regions, the reduction is generally more conservative.

The safe sphere is defined as $B(\Theta_c,r)$ with center $\Theta_c$ and radius $r$. An ideal region would be such that $r$ is small and the center is close to $\hat\Theta^{\lambda,\|\cdot\|_{\text{sgl}}}$. Using this safe region, the GAP safe rules at $\lambda_{k+1}$ are derived as, for a variable $i$ and group $g$,
\begin{alignat*}{3}
    \text{Variable screening:}& \;|X_i^\top \Theta_c| + r\|X_i\|_2 < \tau \; &\implies \;  &\hat\beta_i(\lambda_{k+1}) = 0. \\
    \text{Group screening:}& \;\mathcal{T}_g <(1-\alpha)\sqrt{p_g} \; &\implies \; &\hat{\beta}^{(g)}(\lambda_{k+1}) \equiv \boldsymbol{0},
\end{alignat*}
where 
\begin{equation*}
 \mathcal{T}_g=
    \begin{cases}
        \|S(X_g^\top \Theta_c,\alpha)\| + r\|X_g\|, & \text{if } \|X_g^\top \Theta_c \|_\infty > \alpha, \\
        (\|X_g^\top\Theta_c\|_\infty + r\|X_g\| - \alpha)_+, & \text{otherwise}.
    \end{cases}
\end{equation*}
The center $\Theta_c$ and the radius $r$ are derived using the duality gap and are calculated at iteration $t$ in an iterative algorithm as
\begin{equation*}
  \Theta_t(\beta_{(t)}) = \frac{y-\mathbf{X}\beta_{(t)}}{\max(\lambda_{k+1},\|X^\top (y-\mathbf{X}\beta_{(t)})\|_\text{sgl}^* )}, \;\;\;  r_t(\beta_{(t)},\Theta_t) = \sqrt{\frac{2P_{\lambda_{k+1},\alpha}(\beta_{(t)}) - D_{\lambda_{k+1}}(\Theta_t)}{\lambda_{k+1}^2}}, 
\end{equation*}
where $P_{\lambda,\alpha}$ and $D_\lambda$ are the primal and dual objectives, and $\beta_{(t)}$ is the primal value at iteration $t$. The radius and center are expensive to evaluate, so are calculated only every $10$ iterations \citep{Ndiaye2016GAPLasso}.

The above formulation combines both dynamic and sequential screening. The method can also be implemented using just sequential screening, in which the primal values used in the calculation of the center and radius are from $\lambda_{k}$.

For both the GAP safe rules and DFR, theoretically it would be possible to exactly identify the active sets, but both instead require approximations. While GAP safe has different implementations, we present the best performing versions in our studies. 

\section{Synthetic Data Analysis}\label{appendix:syntheticresuts}

This section complements Section \ref{section:results_synthetic} by providing further information about the simulation set-up and additional results for the synthetic data. Additional tables and figures are provided that further showcase the effectiveness of DFR, including under a logistic model (Appendix \ref{appendix:results_log}). 

\textit{Notes:} The GAP safe methods are only applicable under linear regression and caused computational issues under uneven groups. Given their poor performance in the simulations considered (Figures \ref{fig:figure-group-3-if}, \ref{fig:figure-group-3-ip}, and \ref{fig:figure-group-3-i-ip-if}), they were excluded from the remaining simulations. Additionally, note that an intercept was only applied for linear models. Applying an intercept centers the response data, which is not applicable to a binary response. As we are interested in computational cost, not predictive performance, adding an intercept to logistic models provides no benefit.

\subsection{Metrics}
The following metrics are shown in the tables in the Appendix:
\begin{itemize}
    \item $\mathcal{A}_v, \mathcal{A}_g$: the number of active variables/groups.
    \item $\mathcal{C}_v,\mathcal{C}_g$: the number of variables/groups in the candidate sets.
    \item $\mathcal{O}_v,\mathcal{O}_g$: the number of variables/groups used in the optimization process. As per Algorithms \ref{alg:sgs_framework} and \ref{alg:adap_sgl_framework}, $\mathcal{O}_v = \mathcal{C}_v\cup \mathcal{A}_v$. However, $\mathcal{O}_g$ is not produced as $\mathcal{O}_g = \mathcal{C}_g\cup \mathcal{A}_g$. Instead, $\mathcal{O}_g$ are the groups for which there are variables present in $\mathcal{O}_v$ to give a measure of the number groups used in the optimization.
    \item $\mathcal{K}_v,\mathcal{K}_g$: the number of variable/group KKT violations. DFR only checks for variable violations and sparsegl only checks for group violations.
    \item $\mathcal{O}_v\mathbin{/}\mathcal{A}_v$ and $\mathcal{O}_g\mathbin{/}\mathcal{A}_g$: the proportion of variables/groups used in the optimization against the number active. Defines how efficient the rules are. A low value is best. 
    \item $\mathcal{O}_v\mathbin{/}p$ and $\mathcal{O}_g\mathbin{/}m$: the variable/group input proportion, as defined in Section \ref{sec:results}. A low value is best.
    \item $\ell_2$ distance to no screen: $\ell_2$ distance from the fitted values obtained with screening to without.
    \item IF: the improvement factor, as defined in Section \ref{sec:results}.
\end{itemize}
\newpage
\subsection{Setup}\label{appendix:set_up_info}
\begin{table}[!h]
    \centering
    \caption{Default model, data, and algorithm parameters for the synthetic and real data analyses.}
    \label{tbl:appendix_model_data_simulation}
    \begin{tabular}{@{}llcc@{}}
        \toprule
        \textbf{Category} & \textbf{Parameter} & \multicolumn{2}{c}{\textbf{Values}} \\
        \cmidrule(lr){3-4}
         &  & {\textbf{Synthetic}}& {\textbf{Real}} \\
        \midrule
        \multicolumn{4}{@{}l}{\textbf{Model}} \\
        \midrule
        & $\alpha$  & $0.95$ & $0.95$ \\
        & $b_1=b_2$ (aSGL only) & $0.1$ & $0.1$ \\
        & Path length ($l$) & $50$ & $100$ \\
        & Path termination ($\lambda_l$) & $0.1\lambda_1$ & $0.2\lambda_1$ \\
        &Path shape & Log-linear & Log-linear\\
        \midrule
        \multicolumn{4}{@{}l}{\textbf{Data}} \\
        \midrule
        & $p$ & $1000$ & -\\
        & $n$ & $200$ & - \\
        & $m$ (uneven cases) & $22$ & - \\
        & $m$ (even cases) & $50$ & - \\
        &Group sizes (uneven cases) & $[3,100]$ & - \\
        &Group sizes (even cases) & $20$ & - \\
        & Signal $\beta$ (signal strength of zero) & $\mathcal{N}(0,4)$ & - \\
        & Variable sparsity & $0.2$ & - \\
        & Group sparsity & $0.2$ & - \\
        & Correlation ($\rho$) & $0.3$ & -\\
        & Noise ($\epsilon$) & $\mathcal{N}(0,1)$& - \\  
        \midrule
        \multicolumn{4}{@{}l}{\textbf{Algorithm (ATOS/BCD)}} \\
        \midrule
        & Maximum iterations & $5000$ & $10000$ \\
        & Backtracking (ATOS only) & $0.7$ & $0.7$ \\
        & Maximum backtracking iterations (ATOS only)  & $100$ & $100$ \\
        &Convergence tolerance & $10^{-5}$& $10^{-5}$ \\
        &Standardization & $\ell_2$ & $\ell_2$ \\
        &Intercept & Yes for linear & Yes for linear \\
        &Warm starts & Yes & Yes \\
        \bottomrule
    \end{tabular}
\end{table}

\clearpage
\newpage
\subsection{Runtime Breakdown}\label{appendix:comp-breakdown}
The runtime breakdowns for two cases are presented to illustrate the computational cost of screening.
\begin{figure}[h!]
\centering
\begin{minipage}{0.48\textwidth}
\centering
\captionof{table}{Figure \ref{fig:figure-group-3-if} runtime breakdown.}
\label{tbl:comp-breakdown-1}
\begin{tabular}{@{}lcc@{}}
\toprule
\textbf{Component} & \textbf{DFR-SGL} & \textbf{DFR-aSGL} \\
\midrule
Fitting algorithm          & 88\%     & 86\% \\
$\epsilon$-norm evaluation & 3.9\%      & 3.6\%  \\
Group screening            & 3.9\%      & 3.6\%  \\
Variable screening         & 0.01\%   & 0.01\% \\
KKT checks                 & 0.6\%   & 0.6\%  \\
\bottomrule
\end{tabular}
\end{minipage}
\hfill
\begin{minipage}{0.48\textwidth}
\centering
\captionof{table}{\textit{scheetz} runtime breakdown (Figure \ref{fig:real_data_bar_standardized}).}
\label{tbl:comp-breakdown-2}
\begin{tabular}{@{}lcc@{}}
\toprule
\textbf{Component} & \textbf{DFR-SGL} & \textbf{DFR-aSGL} \\
\midrule
Fitting algorithm          & 77\% & 65\% \\
$\epsilon$-norm evaluation & 0.46\%   & 0.57\%  \\
Group screening            & 0.46\%   & 0.58\%  \\
Variable screening         & 0.01\%   & 0.02\%  \\
KKT checks                 & 0.2\%    & 0.3\%  \\
\bottomrule
\end{tabular}
\end{minipage}
\end{figure}

\subsection{Additional Results for the Linear Model}\label{appendix:lin_add_results}

\begin{figure}[H]
  \centering
    \includegraphics[width=.6\columnwidth]{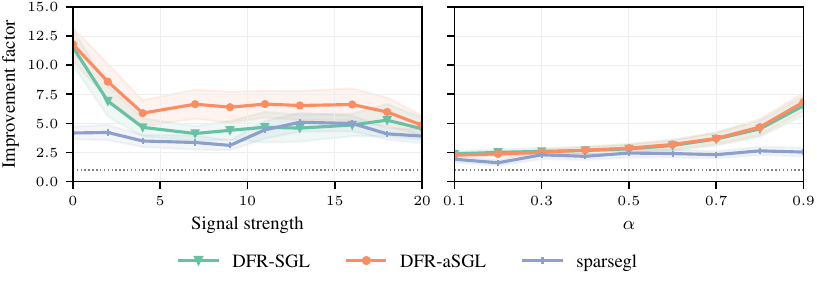}
  \caption{The improvement factor for the strong rules applied to synthetic data, under the linear model, as a function of the signal strength (left) and $\alpha$ (right), with $95\%$ confidence intervals. %The input proportion is shown in Figure \ref{fig:appendix_case_2_input_propn}.
  }
  \label{fig:figure-group-6-if}
\end{figure}

\begin{figure}[H]
  \centering
    \includegraphics[width=.6\columnwidth]{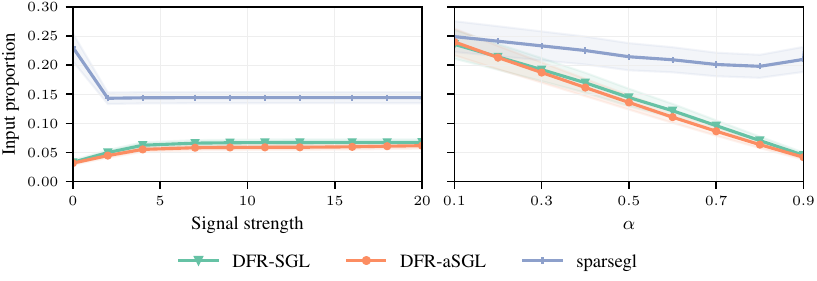}
  \caption{The input proportion for the strong rules applied to synthetic data, under the linear model, as a function of the signal strength (left) and $\alpha$ (right), with $95\%$ confidence intervals. %The input proportion is shown in Figure \ref{fig:appendix_case_2_input_propn}.
  }
  \label{fig:figure-group-6-ip}
\end{figure}
\begin{figure}[H]
  \centering
    \includegraphics[width=.6\columnwidth]{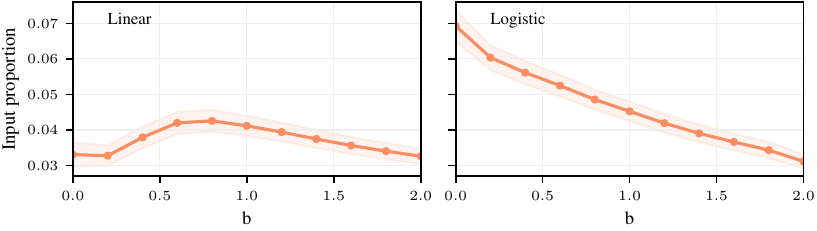}
  \caption{The input proportion of DFR-aSGL under different weights $b_1=b_2$, shown for the linear (left) and logistic (right) models, with $95\%$ confidence intervals.}
  \label{fig:figure-group-4-ip}
\end{figure}
\subsubsection{Tables for GAP safe simulations}
\begin{table}[H]
\caption{Group screening metrics corresponding to the GAP safe simulations (Figures \ref{fig:figure-group-3-if}, \ref{fig:figure-group-3-ip}, and \ref{fig:figure-group-3-i-ip-if}) averaged over all cases and path points, shown with standard errors.}
\vskip 0.15in
\begin{center}
\resizebox{\textwidth}{!}{
\begin{small}
\begin{sc}
\begin{tabular}{lrrrrrr}
\toprule
 & \multicolumn{4}{c}{Cardinality} &\multicolumn{2}{c}{Input proportion}  \\
\cmidrule(lr){2-5}\cmidrule(lr){6-7} 
Method &$\mathcal{A}_g$ & $\mathcal{C}_g$ &  $\mathcal{O}_g$& $\mathcal{K}_g$ &$\mathcal{O}_g \mathbin{/} \mathcal{A}_g$ & $\mathcal{O}_g \mathbin{/} m$  \\
\midrule
DFR-aSGL & $7.86\pm0.02$ & $10.00\pm0.03$ & $10.00\pm0.03$ & -- & $1.1751\pm7\times 10^{-4}$ & $0.2001\pm5\times 10^{-4}$ \\ 
  DFR-SGL & $8.16\pm0.02$ & $10.47\pm0.02$ & $10.47\pm0.02$ & -- & $1.1937\pm8\times 10^{-4}$ & $0.2094\pm5\times 10^{-4}$ \\ 
  DFR-SGL BCD & $8.45\pm0.02$ & $10.07\pm0.02$ & $10.07\pm0.02$ & -- & $1.1893\pm8\times 10^{-4}$ & $0.2015\pm4\times 10^{-4}$ \\ 
  sparsegl & $8.16\pm0.02$ & $11.72\pm0.04$ & $11.72\pm0.04$ & $3\times 10^{-5}\pm2\times 10^{-5}$ & $1.5837\pm0.0111$ & $0.2344\pm7\times 10^{-4}$ \\ 
  GAP sequential & $8.55\pm0.02$ & $10.20\pm0.02$ & $10.20\pm0.02$ & -- & $1.1934\pm0.0018$ & $0.2040\pm4\times 10^{-4}$ \\ 
  GAP dynamic & $8.55\pm0.02$ & $10.20\pm0.02$ & $10.20\pm0.02$ & -- & $1.1917\pm0.0017$ & $0.2039\pm4\times 10^{-4}$ \\ 
      \arrayrulecolor{black}\bottomrule
\end{tabular}
\label{tbl:appendix_gap_sims_grp}
\end{sc}
\end{small}
}
\end{center}
\vskip -0.1in
\end{table}

\begin{table}[H]
\caption{Variable screening metrics corresponding to the GAP safe simulations (Figures \ref{fig:figure-group-3-if}, \ref{fig:figure-group-3-ip}, and \ref{fig:figure-group-3-i-ip-if}) averaged over all cases and path points, shown with standard errors.}
\vskip 0.15in
\begin{center}
\resizebox{\textwidth}{!}{
\begin{small}
\begin{sc}
\begin{tabular}{lrrrrrr}
\toprule
 & \multicolumn{4}{c}{Cardinality} &\multicolumn{2}{c}{Input proportion}  \\
\cmidrule(lr){2-5}\cmidrule(lr){6-7} 
Method &$\mathcal{A}_v$ & $\mathcal{C}_v$ &  $\mathcal{O}_v$& $\mathcal{K}_v$ &$\mathcal{O}_v \mathbin{/} \mathcal{A}_v$ & $\mathcal{O}_v \mathbin{/} p$  \\
\midrule
DFR-aSGL & $61.77\pm0.22$ & $36.70\pm0.34$ & $95.99\pm0.42$ & $0.0092\pm3\times 10^{-4}$ & $1.3464\pm0.0011$ & $0.0960\pm4\times 10^{-4}$ \\ 
  DFR-SGL & $65.37\pm0.22$ & $38.30\pm0.34$ & $101.07\pm0.43$ & $0\pm0$ & $1.3576\pm0.0011$ & $0.1011\pm4\times 10^{-4}$ \\ 
  DFR-SGL BCD & $72.69\pm0.25$ & $23.35\pm0.09$ & $93.11\pm0.32$ & $0.0250\pm0.0054$ & $1.3281\pm0.0011$ & $0.0931\pm3\times 10^{-4}$ \\ 
  sparsegl & $65.41\pm0.22$ & $234.39\pm0.70$ & $234.39\pm0.70$ & -- & $11.5604\pm0.2001$ & $0.2344\pm7\times 10^{-4}$ \\ 
  GAP sequential & $73.01\pm0.25$ & $93.85\pm0.33$ & $93.85\pm0.33$ & -- & $1.3292\pm0.0022$ & $0.0938\pm3\times 10^{-4}$ \\ 
  GAP dynamic & $73.01\pm0.25$ & $93.38\pm0.33$ & $93.38\pm0.33$ & --& $1.2927\pm0.0020$ & $0.0934\pm3\times 10^{-4}$ \\ 
  \arrayrulecolor{black}\bottomrule
\end{tabular}
\label{tbl:appendix_gap_sims_var}
\end{sc}
\end{small}
}
\end{center}
\vskip -0.1in
\end{table}

  \begin{table}[H]
\caption{Model fitting metrics corresponding to the GAP safe simulations (Figures \ref{fig:figure-group-3-if}, \ref{fig:figure-group-3-ip}, and \ref{fig:figure-group-3-i-ip-if}) averaged over all cases and path points, shown with standard errors. The timing results are the average time taken to evaluate the full path on a dataset.}
\vskip 0.15in
\begin{center}
\resizebox{\textwidth}{!}{
\begin{small}
\begin{sc}
\begin{tabular}{lrrrrrrrr}
\toprule
 & \multicolumn{3}{c}{Timings} &\multicolumn{2}{c}{Iterations}&$\ell_2$ distance&\multicolumn{2}{c}{Failed convergence}  \\
\cmidrule(lr){2-4}\cmidrule(lr){5-6} \cmidrule(lr){8-9} 
Method & No screen (s) & Screen (s) &  IF & No screen & Screen&to no screen& No screen & Screen  \\
\midrule
DFR-aSGL & $659.89\pm6.50$ & $154.92\pm2.39$ & $7.01\pm0.15$ & $271.60\pm2.36$ & $174.55\pm2.49$ & $4\times 10^{-4}\pm3\times 10^{-6}$ & $0\pm0$ & $0\pm0$ \\ 
  DFR-SGL & $685.49\pm6.27$ & $157.1\pm2.24$ & $7.98\pm0.20$ & $286.13\pm2.04$ & $185.47\pm2.43$ & $4\times 10^{-4}\pm3\times 10^{-6}$ & $0\pm0$ & $0\pm0$ \\ 
  DFR-SGL BCD & $131.46\pm1.12$ & $32.02\pm0.55$ & $5.55\pm0.08$ & $252.26\pm1.94$ & $139.39\pm1.15$ & $2\times 10^{-7}\pm1\times 10^{-8}$ & $0\pm0$ & $0\pm0$ \\ 
  sparsegl & $685.49\pm6.27$ & $275.37\pm4.43$ & $3.44\pm0.08$ & $286.13\pm2.04$ & $278.18\pm2.28$ & $4\times 10^{-4}\pm3\times 10^{-6}$ & $0\pm0$ & $0\pm0$ \\ 
  GAP sequential & $0.11\pm 3 \times 10^{-3}$ & $0.11\pm 3 \times 10^{-3}$ & $0.98\pm0.01$ & -- & -- & -- & --&  \\ 
  GAP dynamic & $0.11\pm 3 \times 10^{-3}$ & $0.11\pm 3 \times 10^{-3}$ & $1.00\pm0.01$ & -- & -- & -- &-- & -- \\ 
   \arrayrulecolor{black}\bottomrule
\end{tabular}
\label{tbl:appendix_gap_sims_model}
\end{sc}
\end{small}
}
\end{center}
\vskip -0.1in
\end{table}
\subsubsection{Tables for other simulations}
\begin{table}[H]
\caption{Group screening metrics corresponding to the other linear model simulations (Figures \ref{fig:figure-group-4-ip-if}, \ref{fig:figure-group-2-if}, \ref{fig:figure-group-2-ip}, \ref{fig:figure-group-6-if}, and \ref{fig:figure-group-6-ip} and Table \ref{tbl:results_inter_cv}) averaged over all cases and path points, shown with standard errors.}
\vskip 0.15in
\begin{center}
\resizebox{\textwidth}{!}{
\begin{small}
\begin{sc}
\begin{tabular}{lrrrrrr}
\toprule
 & \multicolumn{4}{c}{Cardinality} &\multicolumn{2}{c}{Input proportion}  \\
\cmidrule(lr){2-5}\cmidrule(lr){6-7} 
Method &$\mathcal{A}_g$ & $\mathcal{C}_g$ &  $\mathcal{O}_g$& $\mathcal{K}_g$ &$\mathcal{O}_g \mathbin{/} \mathcal{A}_g$ & $\mathcal{O}_g \mathbin{/} m$  \\
\midrule
DFR-aSGL & $8.46\pm0.03$ & $10.55\pm0.04$ & $10.54\pm0.04$ & -- & $1.1414\pm5\times 10^{-4}$ & $0.1968\pm4\times 10^{-4}$ \\ 
  DFR-SGL & $8.51\pm0.03$ & $11.16\pm0.04$ & $11.16\pm0.04$ & -- & $1.1792\pm6\times 10^{-4}$ & $0.2084\pm4\times 10^{-4}$ \\ 
  sparsegl & $8.51\pm0.03$ & $10.25\pm0.04$ & $10.25\pm0.04$ & $8\times 10^{-5}\pm2\times 10^{-5}$ & $1.2191\pm0.0026$ & $0.2083\pm4\times 10^{-4}$ \\ 
      \arrayrulecolor{black}\bottomrule
\end{tabular}
\label{tbl:appendix_other_sims_grp}
\end{sc}
\end{small}
}
\end{center}
\vskip -0.1in
\end{table}

\begin{table}[H]
\caption{Varible screening metrics corresponding to the other linear model simulations (Figures \ref{fig:figure-group-4-ip-if}, \ref{fig:figure-group-2-if}, \ref{fig:figure-group-2-ip}, \ref{fig:figure-group-6-if}, and \ref{fig:figure-group-6-ip} and Table \ref{tbl:results_inter_cv}) averaged over all cases and path points, shown with standard errors.}
\vskip 0.15in
\begin{center}
\resizebox{\textwidth}{!}{
\begin{small}
\begin{sc}
\begin{tabular}{lrrrrrr}
\toprule
 & \multicolumn{4}{c}{Cardinality} &\multicolumn{2}{c}{Input proportion}  \\
\cmidrule(lr){2-5}\cmidrule(lr){6-7} 
Method &$\mathcal{A}_v$ & $\mathcal{C}_v$ &  $\mathcal{O}_v$& $\mathcal{K}_v$ &$\mathcal{O}_v \mathbin{/} \mathcal{A}_v$ & $\mathcal{O}_v \mathbin{/} p$  \\
\midrule
DFR-aSGL & $50.95\pm0.13$ & $29.90\pm0.15$ & $78.92\pm0.22$ & $0.0297\pm4\times 10^{-4}$ & $1.4805\pm0.0011$ & $0.0634\pm2\times 10^{-4}$ \\ 
  DFR-SGL & $54.40\pm0.13$ & $32.78\pm0.15$ & $85.16\pm0.22$ & $4\times 10^{-6}\pm4\times 10^{-6}$ & $1.5057\pm0.0014$ & $0.0676\pm2\times 10^{-4}$ \\ 
  sparsegl & $54.42\pm0.13$ & $331.75\pm0.86$ & $331.75\pm0.86$ & -- & $11.4845\pm0.0882$ & $0.2205\pm5\times 10^{-4}$ \\ 
  \arrayrulecolor{black}\bottomrule
\end{tabular}
\label{tbl:appendix_other_sims_var}
\end{sc}
\end{small}
}
\end{center}
\vskip -0.1in
\end{table}

  \begin{table}[H]
\caption{Model fitting metrics corresponding to the other linear model simulations (Figures \ref{fig:figure-group-4-ip-if}, \ref{fig:figure-group-2-if}, \ref{fig:figure-group-2-ip}, \ref{fig:figure-group-6-if}, and \ref{fig:figure-group-6-ip} and Table \ref{tbl:results_inter_cv}) averaged over all cases and path points, shown with standard errors. The timing results are the average time taken to evaluate the full path on a dataset.}
\vskip 0.15in
\begin{center}
\resizebox{\textwidth}{!}{
\begin{small}
\begin{sc}
\begin{tabular}{lrrrrrrrr}
\toprule
 & \multicolumn{3}{c}{Timings} &\multicolumn{2}{c}{Iterations}&$\ell_2$ distance&\multicolumn{2}{c}{Failed convergence}  \\
\cmidrule(lr){2-4}\cmidrule(lr){5-6} \cmidrule(lr){8-9} 
Method & No screen (s) & Screen (s) &  IF& No screen & Screen&to no screen& No screen & Screen  \\
\midrule
DFR-aSGL & $723.99\pm34.04$ & $77.00\pm2.23$ & $11.19\pm0.41$ & $415.87\pm7.41$ & $227.93\pm3.51$ & $7\times 10^{-5}\pm5\times 10^{-7}$ & $0\pm0$ & $0\pm0$ \\ 
  DFR-SGL & $251.59\pm3.93$ & $51.73\pm0.89$ & $9.15\pm0.15$ & $413.10\pm7.12$ & $240.78\pm3.78$ & $6\times 10^{-5}\pm5\times 10^{-7}$ & $0\pm0$ & $0\pm0$ \\ 
  sparsegl & $251.59\pm3.93$ & $138.18\pm5.62$ & $3.59\pm0.05$ & $413.10\pm7.12$ & $353.66\pm5.79$ & $6\times 10^{-5}\pm5\times 10^{-7}$ & $0\pm0$ & $0\pm0$ \\ 
   \arrayrulecolor{black}\bottomrule
\end{tabular}
\label{tbl:appendix_other_sims_model}
\end{sc}
\end{small}
}
\end{center}
\vskip -0.1in
\end{table}
\subsubsection{Cross-validation}
\begin{table}[H]
  \centering
  \captionof{table}{The improvement factor for the strong rules applied to synthetic data, under the linear and logistic models, with 10-fold CV, with standard errors.}
\label{tbl:results_cv}
  \begin{tabular}{lrr}
    \toprule
   Method & Linear & Logistic    \\
    \midrule
    DFR-aSGL    &$3.9\pm 0.2$  &$2.3\pm 0.1$   \\
    DFR-SGL       &$4.2\pm 0.3$    &$2.6\pm 0.1$   \\
    sparsegl & $2.0\pm 0.2$   & $2.1\pm 0.1$   \\
    \bottomrule
  \end{tabular}
\end{table}

\subsection{Interaction Models} \label{appendix:interactions}
\begin{figure}[H]
  \centering
    \includegraphics[width=.6\columnwidth]{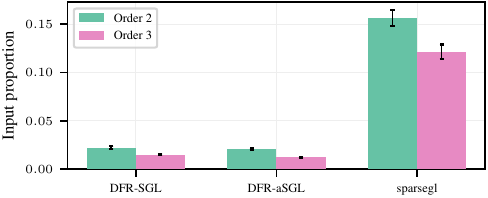}
  \caption{The input proportion for the strong rules applied to synthetic interaction data, under the linear model, with standard errors. The parameters of the data were set as $p=400,n=80$, and $m=52$ groups of sizes in $[3,15]$. The interaction input dimensionality was $p_{O_2} = 2111$ and $p_{O_3}=7338$, with no interaction hierarchy imposed. The sparsity proportion of interaction variables was set to $0.3$ (with the same signal as the marginal effects).}
\label{fig:appendix_case_1_i_input_propn}
\end{figure}
\newpage

\subsection{Results for the Logistic Model}\label{appendix:results_log}
The data input components $\mathbf{X}$, $\beta$, and $\epsilon$ for the logistic model were generated as for the linear models. The class probabilities for the response were calculated using $\sigma(\mathbf{X}\beta + \epsilon)$, where $\sigma$ is the sigmoid function.
\begin{table}[H]
  \centering
  \captionof{table}{The improvement factor for the strong rules applied to synthetic interaction data, under the logistic model, with standard errors. The parameters of the data were set as $p=400,n=80$, and $m=52$ groups of sizes in $[3,15]$. The interaction input dimensionality was $p_{O_2} = 2111$ and $p_{O_3}=7338$, with no interaction hierarchy imposed. The sparsity proportion of interaction variables was set to $0.3$ (with the same signal as the marginal effects).}
\label{tbl:results_inter_cv_log}
  \begin{tabular}{lrr}
    \toprule
     & \multicolumn{2}{c}{Interaction}   \\
    \cmidrule(lr){2-3}
                            Method & Order 2 & Order 3     \\
    \midrule
    DFR-aSGL                     & $6.7\pm 0.4$ & $12.2\pm 0.4$   \\
    DFR-SGL                 & $5.8 \pm 0.2$  & $8.3\pm 0.4$      \\
    sparsegl                     & $1.0\pm 0.1$      & $2.1\pm 0.3$  \\
    \bottomrule
  \end{tabular}
\end{table}
\begin{figure}[H]
  \centering
    \includegraphics[width=.6\columnwidth]{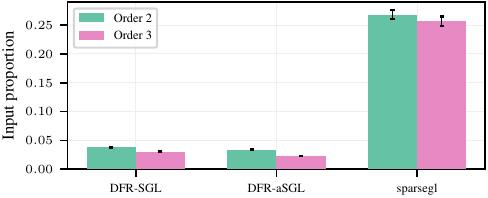}
  \caption{The input proportion for the strong rules applied to synthetic interaction data, under the logistic model, with standard errors. The parameters of the data were set as $p=400,n=80$, and $m=52$ groups of sizes in $[3,15]$. The interaction input dimensionality was $p_{O_2} = 2111$ and $p_{O_3}=7338$, with no interaction hierarchy imposed. The sparsity proportion of interaction variables was set to $0.3$ (with the same signal as the marginal effects).}
\label{fig:appendix_case_1_i_input_propn_log}
\end{figure}

\begin{figure}[H]
  \centering
    \includegraphics[width=.6\columnwidth]{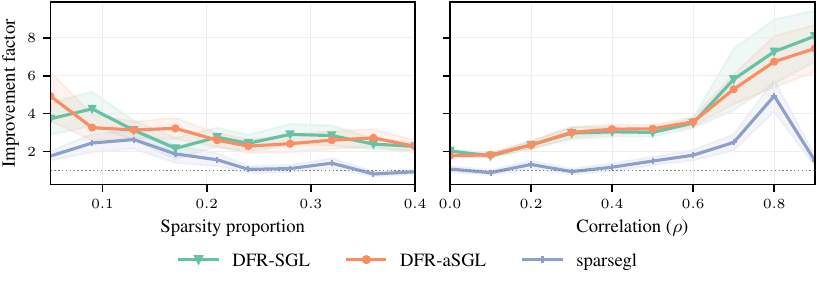}
  \caption{The improvement factor for the strong rules applied to synthetic data, under the logistic model, as a function of the sparsity proportion (left) and data correlation (right), with $95\%$ confidence intervals.
  }
  \label{fig:figure-group-2-log-if}
\end{figure}

\begin{figure}[H]
  \centering
    \includegraphics[width=.6\columnwidth]{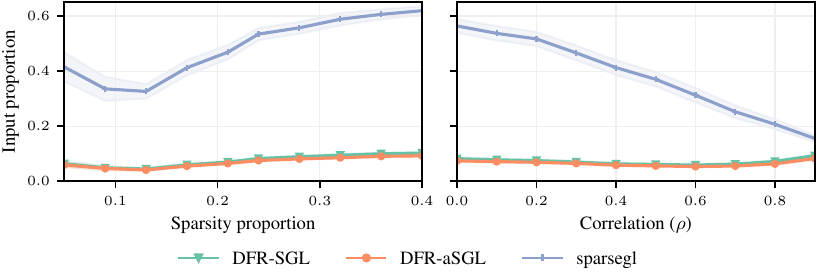}
  \caption{The input proportion for the strong rules applied to synthetic data, under the logistic model, as a function of the sparsity proportion (left) and data correlation (right), with $95\%$ confidence intervals.
  }
  \label{fig:figure-group-2-log-ip}
\end{figure}
\begin{figure}[H]
  \centering
    \includegraphics[width=.6\columnwidth]{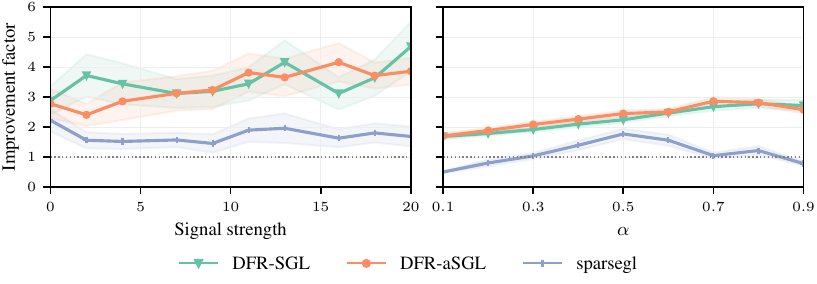}
  \caption{The improvement factor for the strong rules applied to synthetic data, under the logistic model, as a function of the signal strength (left) and $\alpha$ (right), with $95\%$ confidence intervals.%The input proportion is shown in Figure \ref{fig:appendix_case_2_input_propn}.
  }
  \label{fig:figure-group-6-log-if}
\end{figure}

\begin{figure}[H]
  \centering
    \includegraphics[width=.6\columnwidth]{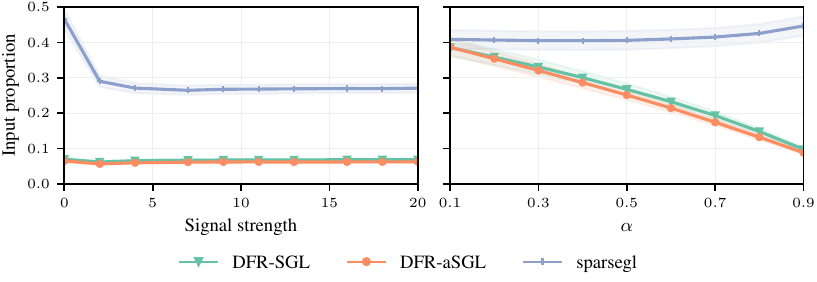}
  \caption{The input proportion for the strong rules applied to synthetic data, under the logistic model, as a function of the signal strength (left) and $\alpha$ (right), with $95\%$ confidence intervals.%The input proportion is shown in Figure \ref{fig:appendix_case_2_input_propn}.
  }
  \label{fig:figure-group-6-log-ip}
\end{figure}
\subsubsection{Tables for logistic simulations}

\begin{table}[H]
\caption{Group screening metrics corresponding to the logistic model simulations (Figures \ref{fig:figure-group-2-log-if}, \ref{fig:figure-group-2-log-ip}, \ref{fig:figure-group-6-log-if}, and \ref{fig:figure-group-6-log-ip} and Table \ref{tbl:results_inter_cv_log}) averaged over all cases and path points, shown with standard errors.}
\vskip 0.15in
\begin{center}
\resizebox{\textwidth}{!}{
\begin{small}
\begin{sc}
\begin{tabular}{lrrrrrr}
\toprule
 & \multicolumn{4}{c}{Cardinality} &\multicolumn{2}{c}{Input proportion}  \\
\cmidrule(lr){2-5}\cmidrule(lr){6-7} 
Method &$\mathcal{A}_g$ & $\mathcal{C}_g$ &  $\mathcal{O}_g$& $\mathcal{K}_g$ &$\mathcal{O}_g \mathbin{/} \mathcal{A}_g$ & $\mathcal{O}_g \mathbin{/} m$  \\
\midrule

DFR-aSGL & $8.08\pm0.01$ & $8.55\pm0.02$ & $8.55\pm0.02$ & --& $1.0292\pm6\times 10^{-4}$ & $0.3528\pm6\times 10^{-4}$ \\ 
  DFR-SGL & $8.56\pm0.02$ & $9.30\pm0.02$ & $9.30\pm0.02$ & -- & $1.0554\pm7\times 10^{-4}$ & $0.3823\pm7\times 10^{-4}$ \\ 
  sparsegl & $8.56\pm0.02$ & $8.87\pm0.02$ & $8.87\pm0.02$ & $5\times 10^{-5}\pm2\times 10^{-5}$ & $1.0808\pm0.0029$ & $0.3697\pm7\times 10^{-4}$ \\ 
      \arrayrulecolor{black}\bottomrule
\end{tabular}
\label{tbl:appendix_log_sims_grp}
\end{sc}
\end{small}
}
\end{center}
\vskip -0.1in
\end{table}

\begin{table}[H]
\caption{Variable screening metrics corresponding to the logistic model simulations (Figures \ref{fig:figure-group-2-log-if}, \ref{fig:figure-group-2-log-ip}, \ref{fig:figure-group-6-log-if}, and \ref{fig:figure-group-6-log-ip} and Table \ref{tbl:results_inter_cv_log}) averaged over all cases and path points, shown with standard errors.}
\vskip 0.15in
\begin{center}
\resizebox{\textwidth}{!}{
\begin{small}
\begin{sc}
\begin{tabular}{lrrrrrr}
\toprule
 & \multicolumn{4}{c}{Cardinality} &\multicolumn{2}{c}{Input proportion}  \\
\cmidrule(lr){2-5}\cmidrule(lr){6-7} 
Method &$\mathcal{A}_v$ & $\mathcal{C}_v$ &  $\mathcal{O}_v$& $\mathcal{K}_v$ &$\mathcal{O}_v \mathbin{/} \mathcal{A}_v$ & $\mathcal{O}_v \mathbin{/} p$  \\
\midrule
DFR-aSGL & $79.76\pm0.24$ & $35.27\pm0.10$ & $111.79\pm0.32$ & $0.0155\pm3\times 10^{-4}$ & $1.5263\pm0.0012$ & $0.1075\pm3\times 10^{-4}$ \\ 
  DFR-SGL & $84.23\pm0.25$ & $40.13\pm0.12$ & $120.84\pm0.34$ & $9\times 10^{-6}\pm7\times 10^{-6}$ & $1.5503\pm0.0013$ & $0.1154\pm3\times 10^{-4}$ \\ 
  sparsegl & $84.26\pm0.25$ & $445.21\pm1.05$ & $445.21\pm1.05$ & -- & $11.6349\pm0.1015$ & $0.4004\pm7\times 10^{-4}$ \\ 
  \arrayrulecolor{black}\bottomrule
\end{tabular}
\label{tbl:appendix_log_sims_var}
\end{sc}
\end{small}
}
\end{center}
\vskip -0.1in
\end{table}

  \begin{table}[H]
\caption{Model fitting metrics corresponding to the logistic model simulations (Figures \ref{fig:figure-group-2-log-if}, \ref{fig:figure-group-2-log-ip}, \ref{fig:figure-group-6-log-if}, and \ref{fig:figure-group-6-log-ip} and Table \ref{tbl:results_inter_cv_log}) averaged over all cases and path points, shown with standard errors. The timing results are the average time taken to evaluate the full path on a dataset.}
\vskip 0.15in
\begin{center}
\resizebox{\textwidth}{!}{
\begin{small}
\begin{sc}
\begin{tabular}{lrrrrrrrr}
\toprule
 & \multicolumn{3}{c}{Timings} &\multicolumn{2}{c}{Iterations}&$\ell_2$ distance&\multicolumn{2}{c}{Failed convergence}  \\
\cmidrule(lr){2-4}\cmidrule(lr){5-6} \cmidrule(lr){8-9} 
Method & No screen (s) & Screen (s) &  IF& No screen & Screen&to no screen& No screen & Screen  \\
\midrule
DFR-aSGL & $132.10\pm7.21$ & $34.75\pm0.66$ & $3.37\pm0.05$ & $125.73\pm2.77$ & $78.95\pm1.10$ & $9\times 10^{-10}\pm2\times 10^{-11}$ & $0\pm0$ & $0\pm0$ \\ 
  DFR-SGL & $103.16\pm3.37$ & $34.27\pm0.56$ & $3.43\pm0.05$ & $141.42\pm2.92$ & $87.88\pm1.30$ & $2\times 10^{-10}\pm8\times 10^{-12}$ & $0\pm0$ & $0\pm0$ \\ 
  sparsegl & $103.16\pm3.37$ & $109.93\pm4.69$ & $1.51\pm0.03$ & $141.42\pm2.92$ & $121.29\pm2.11$ & $2\times 10^{-10}\pm3\times 10^{-12}$ & $0\pm0$ & $0\pm0$ \\ 
   \arrayrulecolor{black}\bottomrule
\end{tabular}
\label{tbl:appendix_log_sims_model}
\end{sc}
\end{small}
}
\end{center}
\vskip -0.1in
\end{table}
\newpage
\clearpage
\newpage
\section{Real Data Analysis}\label{appendix:realdata}

\subsection{Data Description}\label{appendix:real_data_description}
\begin{itemize}
    \item brca1: Gene expression data for breast cancer tissue samples.
    \begin{itemize}
        \item Response (continuous): Gene expression measurements for the BRCA1 gene.
        \item Data matrix: Gene expression measurements for the other genes.
        \item Grouping structure: Variables are grouped via singular value decomposition.
    \end{itemize}
    \item scheetz: Gene expression data in the mammalian eye. 
    \begin{itemize}
        \item Response (continuous): Gene expression measurements for the Trim32 gene.
        \item Data matrix: Gene expression measurements for the other genes.
        \item Grouping structure: Variables are grouped via singular value decomposition.
    \end{itemize}
    \item trust-experts: Survey response data as to how much participants trust experts (e.g., doctors, nurses, scientists) to provide COVID-19 news and information.
    \begin{itemize}
        \item Response (continuous): The trust level of each participant.
        \item Data matrix: Contingency table including factors about participants (e.g., age, gender, ethnicity).
        \item Grouping structure: The factor levels grouped into their original factors.
    \end{itemize}
    \item adenoma: Transcriptome profile data to identify the formation of colorectal adenomas, which are the predominate cause of colorectal cancers. 
    \begin{itemize}
        \item Response (binary): Labels classifying whether the sample came from an adenoma or normal mucosa.
        \item Data matrix: Transcriptome profile measurements.
        \item Grouping structure: Genes were assigned to pathways from all nine gene sets on the Molecular Signature Database.\footnote{downloaded on 08/2024 from \label{note1}\url{gsea-msigdb.org/gsea/msigdb/human/collections.jsp}.}
    \end{itemize}
    \item celiac: Gene expression data of primary leucocytes to identify celiac disease.
    \begin{itemize}
        \item Response (binary): Labels classifying patients into whether they have celiac disease.
        \item Data matrix: Gene expression measurements from the primary leucocytes.
        \item Grouping structure: Genes were mapped to pathways from all nine Molecular Signature Database gene sets.\footnotemark[1]
    \end{itemize}
    \item tumour: Gene expression data of pancreative cancer samples to identify tumorous tissue. 
    \begin{itemize}
        \item Response (binary): Labels classifying whether samples are from normal or tumour tissue.
        \item Data matrix: Gene expression measurements.
        \item Grouping structure: Genes were mapped to pathways from all nine Molecular Signature Database gene sets.\footnotemark[1]
    \end{itemize}
\end{itemize}
\begin{table}[H]
  \caption{Dataset information for the six datasets used in the real data analysis.}
  \label{tbl:appendix_real_dataset_info}
  \centering
\resizebox{\textwidth}{!}{  \begin{tabular}{llllllll}
    \toprule
     Dataset & $p$ & $n$ & $m$ & Group sizes & Type & Source \\
    \midrule
brca1& $17322$ & $536$ & $243$&$[1,6505]$& Linear&\cite{NCI}\tablefootnote{downloaded on 08/2024 from \url{https://iowabiostat.github.io/data-sets/}.\label{reference footnote}} \\
scheetz&  $18975$ & $120$ & $85$&$[1,6274]$& Linear & \cite{scheetz2006RegulationDisease}\textsuperscript{\getrefnumber{reference footnote}}  \\
trust-experts& $101$ & $9759$ & $7$&$[4,51]$& Linear&\cite{Salomon2021TheVaccination}\tablefootnote{downloaded on 08/2024 from \url{https://github.com/dajmcdon/sparsegl}.} \\
adenoma&$18559$ & $64$ & $313$&$[1,741]$& Logistic&\cite{Sabates-Bellver2007TranscriptomeAdenomas}\tablefootnote{downloaded on 08/2024 from \url{https://www.ncbi.nlm.nih.gov/}\label{genetics footnote}.} \\
celiac&$14657$ & $132$ & $276$ & $[1,617]$ & Logistic&\cite{Heap2009ComplexLeucocytes}\textsuperscript{\getrefnumber{genetics footnote}}\\
tumour&$18559$ & $52$ & $313$ & $[1,741]$ & Logistic&\cite{Pei2009FKBP51Akt,Ellsworth2013ContributionAdenocarcinoma,Li2016GeneticCancer}\textsuperscript{\getrefnumber{genetics footnote}}\\
    \bottomrule
  \end{tabular}
  }
\end{table}
\subsection{Additional Results for the Real Data}\label{appendix:real_add_results}
\newpage

\begin{sidewaystable}[htbp]
\begin{table}[H]
\caption{Group screening metrics corresponding to the real data studies (Figures \ref{fig:real_data_bar_standardized}, \ref{fig:appendix_real_data_input_propn}, and \ref{fig:real_data_pathwise})  averaged over all path points, shown with standard errors.}
\vskip 0.15in
\begin{center}
\resizebox{\textwidth}{!}{
\begin{small}
\begin{sc}
\begin{tabular}{llrrrrrr}
\toprule
 && \multicolumn{4}{c}{Cardinality} &\multicolumn{2}{c}{Input proportion}  \\
\cmidrule(lr){3-6}\cmidrule(lr){7-8} 
Method & Dataset&$\mathcal{A}_g$ & $\mathcal{C}_g$ &  $\mathcal{O}_g$& $\mathcal{K}_g$ &$\mathcal{O}_g \mathbin{/} \mathcal{A}_g$ & $\mathcal{O}_g \mathbin{/} m$  \\
\midrule
DFR-aSGL & adenoma & $1.41\pm0.08$ & $1.38\pm0.08$ & $1.42\pm0.08$ & - & $1.0053\pm0.0053$ & $0.0046\pm3\times 10^{-4}$ \\ 
  DFR-SGL & adenoma & $3.30\pm0.21$ & $13.46\pm0.68$ & $13.46\pm0.68$ & - & $4.6121\pm0.2003$ & $0.0430\pm0.0022$ \\ 
  sparsegl & adenoma & $3.32\pm0.21$ & $8.59\pm0.53$ & $8.59\pm0.53$ & $0\pm0$ & $2.6141\pm0.0942$ & $0.0274\pm0.0017$ \\ 
          \arrayrulecolor{black!20}\midrule
  DFR-aSGL & celiac & $19.41\pm1.65$ & $29.94\pm2.45$ & $28.82\pm2.35$ & - & $1.4412\pm0.0273$ & $0.1044\pm0.0085$ \\ 
  DFR-SGL & celiac & $15.35\pm1.45$ & $22.04\pm2.03$ & $22.04\pm2.03$ & -& $1.4367\pm0.0276$ & $0.0799\pm0.0074$ \\ 
  sparsegl & celiac & $15.36\pm1.46$ & $19.32\pm1.84$ & $19.32\pm1.84$ & $0\pm0$ & $1.2415\pm0.0213$ & $0.0700\pm0.0067$ \\ 
          \arrayrulecolor{black!20}\midrule
  DFR-aSGL & brca1 & $3.84\pm0.26$ & $4.17\pm0.33$ & $4.16\pm0.33$ & -& $1.0439\pm0.0121$ & $0.0171\pm0.0013$ \\ 
  DFR-SGL & brca1 & $5.60\pm0.48$ & $6.90\pm0.60$ & $6.90\pm0.60$ &- & $1.2023\pm0.0218$ & $0.0284\pm0.0025$ \\ 
  sparsegl & brca1 & $5.59\pm0.48$ & $6.25\pm0.55$ & $6.25\pm0.55$ & $0\pm0$ & $1.1062\pm0.0176$ & $0.0257\pm0.0022$ \\ 
          \arrayrulecolor{black!20}\midrule
  DFR-aSGL & scheetz & $2.19\pm0.12$ & $2.37\pm0.16$ & $2.39\pm0.16$ & -& $1.0515\pm0.0121$ & $0.0282\pm0.0019$ \\ 
  DFR-SGL & scheetz & $0.61\pm0.08$ & $0.86\pm0.12$ & $0.86\pm0.12$ & - & $1.3537\pm0.0531$ & $0.0101\pm0.0014$ \\ 
  sparsegl & scheetz & $0.61\pm0.08$ & $0.73\pm0.10$ & $0.73\pm0.10$ & $0\pm0$ & $1.1829\pm0.0486$ & $0.0086\pm0.0012$ \\ 
          \arrayrulecolor{black!20}\midrule
  DFR-aSGL & trust-experts & $3.41\pm0.08$ & $3.37\pm0.09$ & $3.41\pm0.08$ & - & $1.0000\pm0.0000$ & $0.4877\pm0.0118$ \\ 
  DFR-SGL & trust-experts & $3.34\pm0.08$ & $3.37\pm0.08$ & $3.37\pm0.08$ & - & $1.0185\pm0.0117$ & $0.4820\pm0.0113$ \\ 
  sparsegl & trust-experts & $3.30\pm0.09$ & $0.04\pm0.02$ & $3.30\pm0.09$ & $0\pm0$ & $1.0000\pm0.0000$ & $0.4719\pm0.0127$ \\ 
          \arrayrulecolor{black!20}\midrule
  DFR-aSGL & tumour & $3.94\pm0.22$ & $4.96\pm0.30$ & $4.98\pm0.30$ & - & $1.2228\pm0.0240$ & $0.0159\pm1\times 10^{-3}$ \\ 
  DFR-SGL & tumour & $5.02\pm0.24$ & $8.80\pm0.38$ & $8.80\pm0.38$ & - & $1.8253\pm0.0357$ & $0.0281\pm0.0012$ \\ 
  sparsegl & tumour & $5.02\pm0.24$ & $6.77\pm0.29$ & $6.77\pm0.29$ & $0\pm0$ & $1.4276\pm0.0351$ & $0.0216\pm9\times 10^{-4}$ \\ 
      \arrayrulecolor{black}\bottomrule
\end{tabular}
\label{tbl:appendix_real_data_grp}
\end{sc}
\end{small}
}
\end{center}
\vskip -0.1in
\end{table}
\end{sidewaystable}

\begin{sidewaystable}
\begin{table}[H]
\caption{Variable screening metrics corresponding to the real data studies (Figures \ref{fig:real_data_bar_standardized}, \ref{fig:appendix_real_data_input_propn}, and \ref{fig:real_data_pathwise})  averaged over all path points, shown with standard errors.}
\vskip 0.15in
\begin{center}
\resizebox{\textwidth}{!}{
\begin{small}
\begin{sc}
\begin{tabular}{llrrrrrr}
\toprule
 && \multicolumn{4}{c}{Cardinality} &\multicolumn{2}{c}{Input proportion}  \\
\cmidrule(lr){3-6}\cmidrule(lr){7-8} 
Method & Dataset&$\mathcal{A}_v$ & $\mathcal{C}_v$ &  $\mathcal{O}_v$& $\mathcal{K}_v$ &$\mathcal{O}_v \mathbin{/} \mathcal{A}_v$ & $\mathcal{O}_v \mathbin{/} p$  \\
\midrule
DFR-aSGL & adenoma & $3.38\pm0.21$ & $0.81\pm0.11$ & $4.15\pm0.30$ & $0.0404\pm0.0199$ & $1.1828\pm0.0259$ & $2\times 10^{-4}\pm2\times 10^{-5}$ \\ 
  DFR-SGL & adenoma & $14.38\pm1.11$ & $61.47\pm2.90$ & $75.51\pm3.85$ & $0\pm0$ & $8.9655\pm0.7776$ & $0.0041\pm2\times 10^{-4}$ \\ 
  sparsegl & adenoma & $14.41\pm1.12$ & $308.64\pm20.81$ & $308.64\pm20.81$ & - & $26.0099\pm1.6361$ & $0.0166\pm0.0011$ \\ 
          \arrayrulecolor{black!20}\midrule
  DFR-aSGL & celiac & $40.13\pm3.85$ & $29.63\pm2.79$ & $68.61\pm6.51$ & $0.1010\pm0.0337$ & $1.6545\pm0.0442$ & $0.0047\pm4\times 10^{-4}$ \\ 
  DFR-SGL & celiac & $37.13\pm3.92$ & $26.11\pm3.05$ & $61.94\pm6.77$ & $0\pm0$ & $1.6187\pm0.0455$ & $0.0042\pm5\times 10^{-4}$ \\ 
  sparsegl & celiac & $37.26\pm3.94$ & $1019.13\pm106.78$ & $1019.13\pm106.78$ & -& $27.6753\pm1.1176$ & $0.0695\pm0.0073$ \\ 
          \arrayrulecolor{black!20}\midrule
  DFR-aSGL & brca1 & $135.75\pm4.69$ & $31.93\pm2.77$ & $165.83\pm6.19$ & $0.0505\pm0.0221$ & $1.1998\pm0.0116$ & $0.0096\pm4\times 10^{-4}$ \\ 
  DFR-SGL & brca1 & $241.69\pm7.34$ & $63.42\pm4.26$ & $301.8\pm8.81$ & $0\pm0$ & $3.9386\pm2.6945$ & $0.0174\pm5\times 10^{-4}$ \\ 
  sparsegl & brca1 & $241.59\pm7.34$ & $6762.31\pm186.59$ & $6762.31\pm186.59$ & -& $95.8171\pm66.017$ & $0.3904\pm0.0108$ \\ 
          \arrayrulecolor{black!20}\midrule
  DFR-aSGL & scheetz & $743.43\pm13.50$ & $281.91\pm13.50$ & $1019.11\pm20.19$ & $0.0202\pm0.0142$ & $1.3678\pm0.0050$ & $0.0537\pm0.0011$ \\ 
  DFR-SGL & scheetz & $501.78\pm60.68$ & $344.92\pm50.24$ & $836.23\pm101.02$ & $0\pm0$ & $1.6688\pm0.0651$ & $0.0441\pm0.0053$ \\ 
  sparsegl & scheetz & $501.49\pm60.66$ & $3030.24\pm385.60$ & $3030.24\pm385.60$ & -& $6.1978\pm0.3472$ & $0.1597\pm0.0203$ \\ 
          \arrayrulecolor{black!20}\midrule
  DFR-aSGL & trust-experts & $4.83\pm0.20$ & $0.17\pm0.04$ & $4.96\pm0.21$ & $0.0404\pm0.0199$ & $1.0231\pm0.0067$ & $0.0491\pm0.0021$ \\ 
  DFR-SGL & trust-experts & $5.10\pm0.28$ & $0.35\pm0.07$ & $5.36\pm0.29$ & $0\pm0$ & $1.0608\pm0.0155$ & $0.0531\pm0.0029$ \\ 
  sparsegl & trust-experts & $5.07\pm0.28$ & $0.27\pm0.14$ & $21.42\pm0.65$ & - & $4.9628\pm0.1448$ & $0.2121\pm0.0065$ \\ 
          \arrayrulecolor{black!20}\midrule
  DFR-aSGL & tumour & $7.40\pm0.58$ & $3.31\pm0.31$ & $10.57\pm0.84$ & $0.0202\pm0.0142$ & $1.3363\pm0.0394$ & $6\times 10^{-4}\pm5\times 10^{-5}$ \\ 
  DFR-SGL & tumour & $10.70\pm0.72$ & $9.87\pm0.50$ & $20.34\pm1.09$ & $0\pm0$ & $2.3432\pm0.1189$ & $0.0011\pm6\times 10^{-5}$ \\ 
  sparsegl & tumour & $10.70\pm0.72$ & $246.8\pm15.39$ & $246.8\pm15.39$ & -& $25.9945\pm0.9193$ & $0.0133\pm8\times 10^{-4}$ \\ 
  \arrayrulecolor{black}\bottomrule
\end{tabular}
\label{tbl:appendix_real_data_var}
\end{sc}
\end{small}
}
\end{center}
\vskip -0.1in
\end{table}
\end{sidewaystable}

\begin{sidewaystable}[htbp]
    \begin{table}[H]
\caption{Model fitting metrics corresponding to the real data studies (Figures \ref{fig:real_data_bar_standardized}, \ref{fig:appendix_real_data_input_propn}, and \ref{fig:real_data_pathwise})  averaged over all path points, shown with standard errors. There are no standard errors for the timing results as the time to calculate the whole path was evaluated.}
\vskip 0.15in
\begin{center}
\resizebox{\textwidth}{!}{
\begin{small}
\begin{sc}
\begin{tabular}{llrrrrrrrr}
\toprule
 && \multicolumn{3}{c}{Timings} &\multicolumn{2}{c}{Iterations}&$\ell_2$ distance&\multicolumn{2}{c}{Failed convergence}  \\
\cmidrule(lr){3-5}\cmidrule(lr){6-7} \cmidrule(lr){9-10} 
Method &Dataset& No screen (s) & Screen (s) &  I.F.& No screen & Screen&to no screen& No screen & Screen  \\
\midrule
DFR-aSGL & adenoma & $8476.04$ & $7.91$ & $1072.10$ & $8903.76\pm270.64$ & $119.18\pm14.69$ & $7\times 10^{-5}\pm6\times 10^{-6}$ & $0.8384\pm0.0372$ & $0\pm0$ \\ 
  DFR-SGL & adenoma & $9017.70$ & $149.10$ & $60.48$ & $9272.90\pm205.31$ & $4374.64\pm286.62$ & $3\times 10^{-5}\pm2\times 10^{-6}$ & $0.8687\pm0.0341$ & $0\pm0$ \\ 
  sparsegl & adenoma & $9017.70$ & $198.36$ & $45.46$ & $9272.90\pm205.31$ & $5140.16\pm390.91$ & $3\times 10^{-5}\pm2\times 10^{-6}$ & $0.8687\pm0.0341$ & $0.0404\pm0.0199$ \\ 
         \arrayrulecolor{black!20}\midrule
  DFR-aSGL & celiac & $1188.14$ & $15.86$ & $74.89$ & $1042.07\pm90.79$ & $120.96\pm14.51$ & $1\times 10^{-6}\pm2\times 10^{-7}$ & $0\pm0$ & $0\pm0$ \\ 
  DFR-SGL & celiac & $1391.78$ & $10.31$ & $134.95$ & $1195.34\pm98.13$ & $75.37\pm9.60$ & $2\times 10^{-7}\pm1\times 10^{-8}$ & $0\pm0$ & $0\pm0$ \\ 
  sparsegl & celiac & $1391.78$ & $16.49$ & $84.40$ & $1195.34\pm98.13$ & $93.29\pm8.86$ & $8\times 10^{-8}\pm3\times 10^{-9}$ & $0\pm0$ & $0\pm0$ \\ 
          \arrayrulecolor{black!20}\midrule
  DFR-aSGL & brca1 & $21889.33$ & $119.16$ & $183.69$ & $1653.45\pm33.90$ & $243.27\pm11.70$ & $2\times 10^{-9}\pm7\times 10^{-10}$ & $0\pm0$ & $0\pm0$ \\ 
  DFR-SGL & brca1 & $22227.01$ & $103.78$ & $214.17$ & $1674.07\pm19.77$ & $334.16\pm13.44$ & $2\times 10^{-12}\pm3\times 10^{-13}$ & $0\pm0$ & $0\pm0$ \\ 
  sparsegl & brca1 & $22227.01$ & $4132.04$ & $5.38$ & $1674.07\pm19.77$ & $1580.73\pm44.23$ & $4\times 10^{-14}\pm1\times 10^{-14}$ & $0\pm0$ & $0\pm0$ \\ 
          \arrayrulecolor{black!20}\midrule
  DFR-aSGL & scheetz & $90569.05$ & $132.20$ & $685.08$ & $6040.81\pm457.60$ & $1030.48\pm90.27$ & $1\times 10^{-6}\pm3\times 10^{-7}$ & $0.5657\pm0.0501$ & $0\pm0$ \\ 
  DFR-SGL & scheetz & $68084.77$ & $2246.13$ & $30.31$ & $1891.21\pm386.40$ & $642.38\pm136.12$ & $3\times 10^{-9}\pm8\times 10^{-10}$ & $0.1818\pm0.0390$ & $0\pm0$ \\ 
  sparsegl & scheetz & $68084.77$ & $6666.65$ & $10.21$ & $1891.21\pm386.40$ & $1890.45\pm386.44$ & $5\times 10^{-22}\pm2\times 10^{-22}$ & $0.1818\pm0.0390$ & $0.1818\pm0.0390$ \\ 
          \arrayrulecolor{black!20}\midrule
  DFR-aSGL & trust-experts & $5.96$ & $3.10$ & $1.92$ & $76.74\pm1.80$ & $85.11\pm3.96$ & $1\times 10^{-11}\pm3\times 10^{-12}$ & $0\pm0$ & $0\pm0$ \\ 
  DFR-SGL & trust-experts & $7.29$ & $3.20$ & $2.28$ & $98.36\pm4.17$ & $92.55\pm5.93$ & $4\times 10^{-11}\pm9\times 10^{-12}$ & $0\pm0$ & $0\pm0$ \\ 
  sparsegl & trust-experts & $7.29$ & $4.52$ & $1.61$ & $98.36\pm4.17$ & $104.24\pm5.15$ & $2\times 10^{-6}\pm2\times 10^{-6}$ & $0\pm0$ & $0\pm0$ \\ 
          \arrayrulecolor{black!20}\midrule
  DFR-aSGL & tumour & $7027.69$ & $86.25$ & $81.48$ & $8783.1\pm265.36$ & $82.75\pm5.61$ & $3\times 10^{-8}\pm9\times 10^{-9}$ & $0.6768\pm0.0472$ & $0\pm0$ \\ 
  DFR-SGL & tumour & $7466.83$ & $89.43$ & $83.49$ & $9272.02\pm224.60$ & $186.27\pm7.64$ & $3\times 10^{-9}\pm2\times 10^{-10}$ & $0.8586\pm0.0352$ & $0\pm0$ \\ 
  sparsegl & tumour & $7466.83$ & $90.03$ & $82.93$ & $9272.02\pm224.60$ & $197.40\pm9.38$ & $2\times 10^{-9}\pm2\times 10^{-10}$ & $0.8586\pm0.0352$ & $0\pm0$ \\ 
   \arrayrulecolor{black}\bottomrule
\end{tabular}
\label{tbl:appendix_real_data_other}
\end{sc}
\end{small}
}
\end{center}
\vskip -0.1in
\end{table}
\end{sidewaystable}

%%%%%%%%%%%%%%%%%%%%%%%%%%%%%%%%%%%%%%%%%%%%%%%%%%%%%%%%%%%%%%%%%%%%%%%%%%%%%%%
%%%%%%%%%%%%%%%%%%%%%%%%%%%%%%%%%%%%%%%%%%%%%%%%%%%%%%%%%%%%%%%%%%%%%%%%%%%%%%%

\end{document}